%% file: aistats_main.tex
\documentclass[twoside]{article}
\usepackage[ruled,vlined]{algorithm2e}

\usepackage{amsthm}
\usepackage{bm}
\usepackage{graphicx}   
\usepackage{array}      
\usepackage{subcaption}
\usepackage{caption}
\usepackage{mathtools}
\usepackage{overpic}
\usepackage{amsmath}
\usepackage{enumitem}
\usepackage{amsfonts}
\usepackage{xcolor}
\usepackage{booktabs}   
\usepackage{multirow}   
\usepackage{makecell}   

\usepackage{tikz}
\usepackage{bm}         
\newtheorem{lemma}{Lemma}
\newtheorem{assumption}{Assumption}
\newtheorem{theorem}{Theorem}

\usepackage[most]{tcolorbox}

\newtcolorbox{mybox}[2][]{%
  enhanced,
  colback=sand,
  colframe=sandborder,
  boxrule=0.4pt,
  arc=2.5mm,
  left=10pt,right=10pt,
  top=4pt,bottom=4pt,
  fonttitle=\bfseries\small,
  coltitle=black,
  boxed title style={empty},
  title={#2},
  #1
}

\definecolor{sand}{RGB}{248,245,235}  
\definecolor{sandborder}{RGB}{220,215,200}

\input{math_commands}

\usepackage[preprint]{aistats2026}

\usepackage[round]{natbib}
\renewcommand{\bibname}{References}
\renewcommand{\bibsection}{\subsubsection*{\bibname}}

\usepackage{hyperref}
\hypersetup{
    colorlinks,
    linkcolor={blue!90!black},
    citecolor={magenta},
    urlcolor={blue!80!black}
}
\usepackage[capitalize,nameinlink,noabbrev]{cleveref} 

\usepackage{amsthm}

\newtheorem{importedtheorem}{Imported Theorem}
\newtheorem{proposition}{Proposition}

\Crefname{importedtheorem}{Imported Theorem}{Imported Theorems}
\Crefname{assumption}{Assumption}{Assumptions}

\makeatletter
\newtheorem*{rep@theorem}{\rep@title}
\newcommand{\newreptheorem}[2]{%
\newenvironment{rep#1}[1]{%
 \def\rep@title{\Cref{##1}, Restated}%
 \begin{rep@theorem}}%
 {\end{rep@theorem}}}
\makeatother

\newreptheorem{theorem}{Theorem}
\newreptheorem{lemma}{Lemma}

\theoremstyle{remark}              
\newtheorem{remark}{Remark}        

\begin{document}

\runningauthor{Aditi Gupta, Raphael A. Meyer, Yotam Yaniv, Elynn Chen, N. Benjamin Erichson}

\twocolumn[

\aistatstitle{Quantifying Epistemic Uncertainty in Diffusion Models}

\aistatsauthor{
Aditi Gupta \And
Raphael A. Meyer \And
Yotam Yaniv \And
Elynn Chen 
}

\aistatsaddress{
Berkeley Lab \& ICSI \And 
UC Berkeley \& ICSI \And
Berkeley Lab \And
New York University 
}

\aistatsauthor{
N.~Benjamin Erichson
}

\aistatsaddress{
Berkeley Lab \& ICSI 
}
]

\begin{abstract}
To ensure high quality outputs, it is important to quantify the epistemic uncertainty of diffusion models.
Existing methods are often unreliable because they mix epistemic and aleatoric uncertainty.
We introduce a method based on Fisher information that explicitly isolates epistemic variance, producing more reliable plausibility scores for generated data. To make this approach scalable, we propose FLARE (Fisher–Laplace Randomized Estimator), which approximates the Fisher information using a uniformly random subset of model parameters. Empirically, FLARE improves uncertainty estimation in synthetic time-series generation tasks, achieving more accurate and reliable filtering than other methods. Theoretically, we bound the convergence rate of our randomized approximation and provide analytic and empirical evidence that last-layer Laplace approximations are insufficient for this task.
\end{abstract}

\input{main}

\section*{Acknowledgments}

NBE would like to acknowledge support from the U.S. Department of Energy, Office of Science, Office of Advanced Scientific Computing Research, EXPRESS: 2025 Exploratory Research for Extreme-Scale Science program, and the Scientific Discovery through
Advanced Computing (SciDAC) program, under Contract
Number DE-AC02-05CH11231 at Berkeley Lab. 
We would also like to acknowledge supported by the Defense Advanced Research Projects Agency (DARPA) under Contract No. HR0011-25-2-0011. The views, opinions, and/or findings expressed are those of the authors and should not be interpreted as representing the official views or policies of the Department of Defense or the U.S. Government.

\bibliography{references}

\input{appendix}

\end{document}

%% file: math_commands.tex
\usepackage{amsmath,amsfonts,bm,bbm}

\def\eqref#1{equation~\ref{#1}}

\def\1{\bm{1}}

\def\eps{{\epsilon}}

\def\vtheta{{\bm{\theta}}}
\def\veta{{\bm{\eta}}}

\def\vc{{\bm{c}}}

\def\vg{{\bm{g}}}

\def\vr{{\bm{r}}}

\def\vu{{\bm{u}}}

\def\vx{{\bm{x}}}

\def\vz{{\bm{z}}}

\def\mA{{\bm{A}}}
\def\mB{{\bm{B}}}
\def\mC{{\bm{C}}}

\def\mH{{\bm{H}}}
\def\mI{{\mathbf{I}}} 
\def\mJ{{\bm{J}}}

\def\mP{{\bm{P}}}

\def\mU{{\bm{U}}}
\def\mV{{\bm{V}}}

\def\mX{{\bm{X}}}

\def\mSigma{{\bm{\Sigma}}}

\DeclareMathAlphabet{\mathsfit}{\encodingdefault}{\sfdefault}{m}{sl}
\SetMathAlphabet{\mathsfit}{bold}{\encodingdefault}{\sfdefault}{bx}{n}

\newcommand{\E}{\mathbb{E}}

\newcommand{\R}{\mathbb{R}}

\newcommand{\Cov}{\mathrm{Cov}}

\DeclareMathOperator*{\argmin}{arg\,min}

\DeclareMathOperator{\tr}{tr}

%% file: main.tex
\section{Introduction}

Diffusion models have become a dominant paradigm for generative modeling, with applications ranging from image synthesis to time-series forecasting~\citep{esser2024scaling,liu2022flow,lipman2022flow,lu2022dpm}.
The stochasticity of the reverse diffusion process accounts for aleatoric variability, suggesting that diffusion models could also serve as a foundation for uncertainty quantification (UQ)~\citep{ho2020ddpm,song2020score,kendall2017uncertainties,hullermeier2021aleatoric}.
However, while aleatoric randomness is intrinsic to sampling, capturing and interpreting \emph{epistemic} uncertainty (arising from uncertainty in model parameters) remains a  challenge~\citep{depeweg2018decomposition,hullermeier2021aleatoric}.
Exact Bayesian inference over the high-dimensional parameter spaces of diffusion networks is computationally infeasible, and common proxies such as sample variance conflate epistemic and aleatoric effects~\citep{shu2024zero,de2025diffusion}.

Several approaches have been proposed to address uncertainty quantification in diffusion models.
One line of work uses last-layer Laplace approximations (LLLA)~\citep{daxberger2021laplace,kou2023bayesdiff,jazbec2025generative}, which are computationally efficient but restrict uncertainty to a small subset of model parameters.
Other methods perturb model parameters directly~\citep{berry2024casting,chan2024estimating,shu2024zero}, capturing richer epistemic structure at the cost of multiple forward passes during inference.
Ensemble and hypernetwork-based approaches approximate the weight posterior more explicitly~\citep{krueger2017bayesian,lakshminarayanan2017deep,maddox2019simple,gal2016dropout}, and can in principle separate epistemic and aleatoric uncertainty.
In practice, however, these methods are sensitive to architectural choices, hyper-priors, and the mechanism by which randomness is injected, often trading predictive fidelity for diversity if not carefully regularized.

A key source of confusion in this literature is the distinction between \emph{posterior predictive} uncertainty and \emph{epistemic} uncertainty.
For example, BayesDiff~\citep{kou2023bayesdiff,jazbec2025generative} estimates posterior predictive uncertainty in diffusion models using Tweedie-style recursions that propagate data-space variance through the reverse process. This approach aggregates multiple sources of variability, including diffusion noise and model output uncertainty, and is well suited for measuring overall sample variability. However, predictive variance alone does not indicate whether a model is uncertain due to lack of knowledge or merely due to stochastic sampling.

\begin{figure}[!b]
    \centering
    \hspace{-0.9cm}
        \begin{overpic}[width=1.0\linewidth,clip]
            {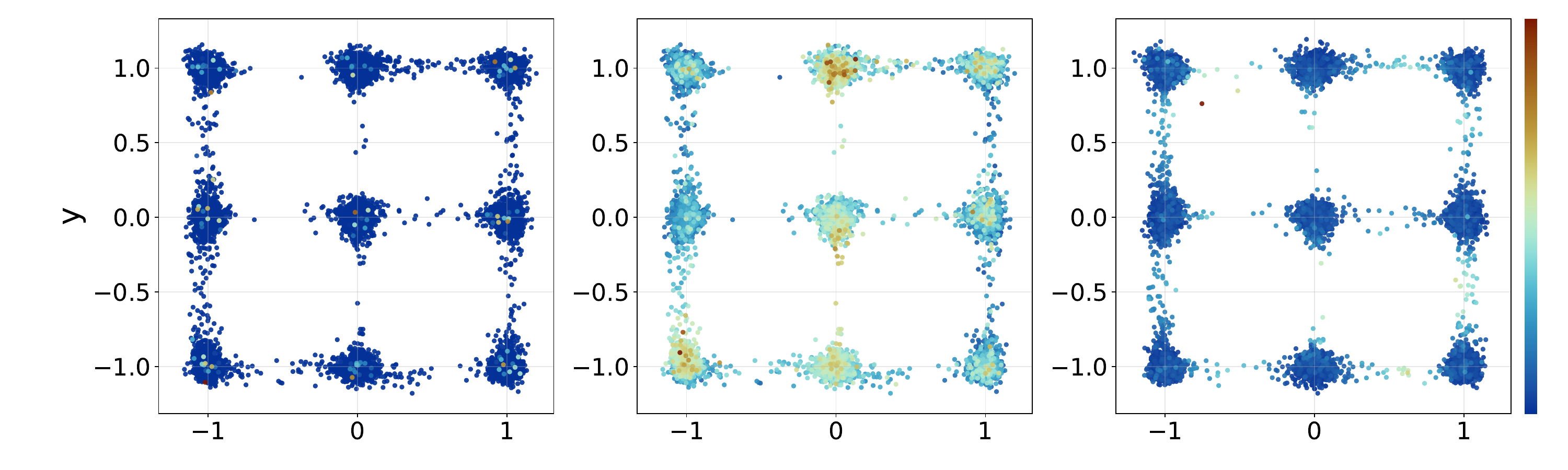}
            \put(16,30){{\scalebox{0.75}{BayesDiff}}}
            \put(48,30){{\scalebox{0.75}{LLLA}}}
            \put(74,30){{\scalebox{0.75}{FLARE (ours)}}}
            \put(99,1){\rotatebox{90}{\scalebox{0.5}{certain}}}
            \put(99,21){\rotatebox{90}{\scalebox{0.5}{ uncertain}}}
        \end{overpic}
        
    \hspace{-0.9cm}
        \begin{overpic}[width=1.0\linewidth,clip]
            {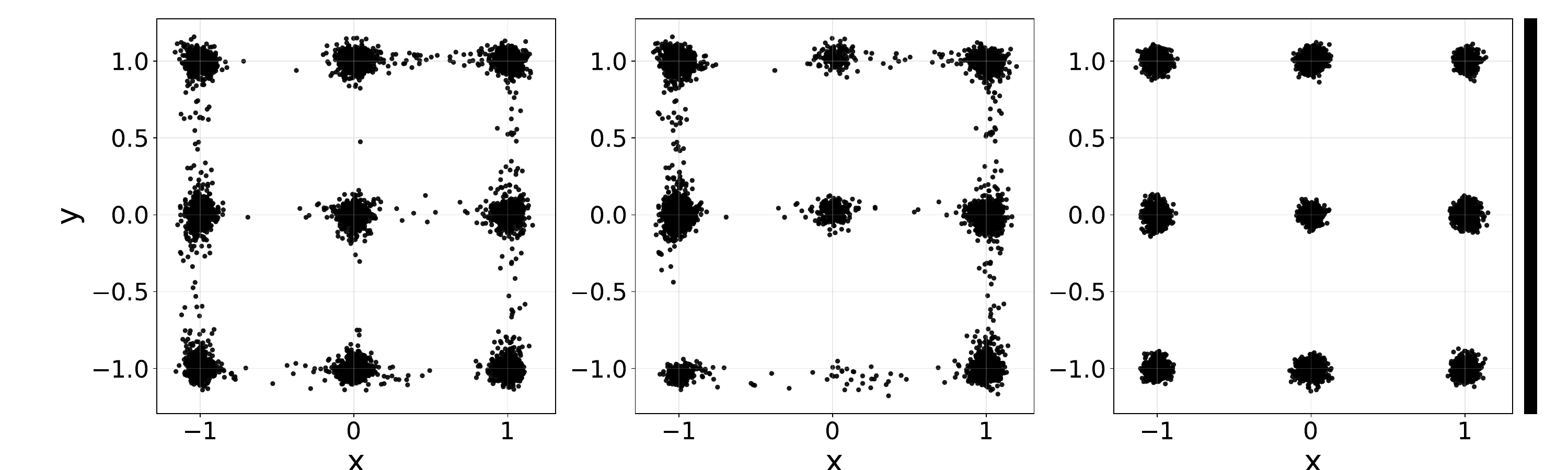}
        \put(97,3){\color{white}\rule{1.5em}{6.5em}} 
        \put(98,8){\rotatebox{90}{\scalebox{0.5}{filtered samples}}}
        \end{overpic}
\caption{Mode interpolation in a 2D Gaussian mixture adapted from~\citet{aithal2024understanding,jazbec2025generative}.
The dataset consists of nine Gaussian modes arranged on a square grid.
\textbf{Top:} uncertainty scores assigned to generated samples by BayesDiff (left), last-layer Laplace (LLLA; middle), and our method (right).
\textbf{Bottom:} the same samples after filtering using a fixed uncertainty threshold.
BayesDiff assigns low uncertainty to samples between modes, while LLLA further suppresses uncertainty due to its restriction to the final layer.
In contrast, our method assigns high epistemic uncertainty in low-density regions between modes, enabling reliable removal of low-confidence samples.}

    \label{fig:grid}
\end{figure}
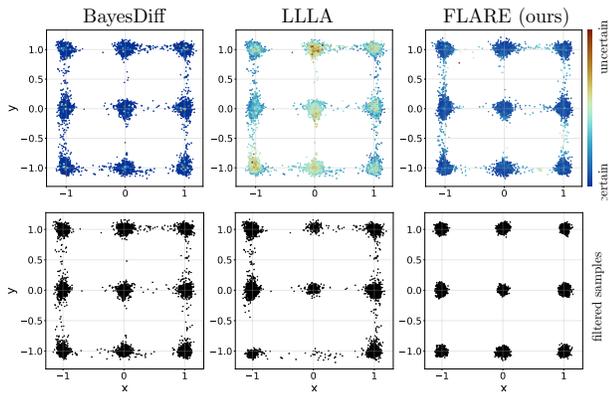

In this work, we focus specifically on \emph{epistemic} uncertainty in diffusion models. Rather than asking how variable the generated samples are, we ask when the model itself is uncertain due to its parameters. Diffusion models already introduce randomness through the reverse process, which accounts for aleatoric variability and is present even when the model is well trained. As a result, measures based on sample variance reflect a mixture of stochastic sampling noise and actual model uncertainty. Interpreting this mixture as a single uncertainty signal makes it difficult to tell whether variability reflects a lack of knowledge or merely randomness in generation.
To address this, we isolate uncertainty arising from the model parameters and track how it propagates through the reverse diffusion process. We project parameter uncertainty through the Jacobian of the denoiser along a realized reverse trajectory, while excluding diffusion noise from the propagated quantity. This yields a trajectory-level view of epistemic uncertainty and clarifies when and where the model is uncertain.

Specifically, we develop a Fisher–Laplace formulation that propagates parameter uncertainty through the reverse diffusion process by using the Jacobian of the denoiser at each step.
While computing the full Fisher information is prohibitively expensive in large models, the common simplification of restricting focus to the last layer discards sensitivity information from earlier representations and leads to systematically miscalibrated epistemic uncertainty.
We therefore introduce FLARE (Fisher–Laplace Randomized Estimator), a scalable algorithm that approximates Fisher–Laplace uncertainty by sampling parameters uniformly at random across all layers of the network.
This \emph{random subnetwork approximation} preserves network-wide sensitivity structure while reducing computational cost.
We show that the resulting estimator is theoretically justified, converges rapidly as the number of sampled parameters increases, and closely tracks the behavior of the full Fisher–Laplace projection. \Cref{fig:grid} illustrates the performance of FLARE on a toy example.

We evaluate FLARE on synthetic time-series generation tasks designed to probe multimodality, extrapolation, and low-density regions.
We compare against BayesDiff~\citep{kou2023bayesdiff}, parameter-perturbation methods such as HyperDM~\citep{chan2024estimating}, and a version of FLARE that uses the last-layer approximation instead of a random subnetwork.
Across all tasks, our method recovers structured epistemic uncertainty in regions where models switch modes, traverse sparse areas, or move off distribution.
In contrast, BayesDiff conflates aleatoric and epistemic effects, LLLA suppresses uncertainty arising from earlier layers, and parameter-perturbation methods incur additional computational overhead while mixing parameter variability with diffusion noise.
Overall, FLARE provides a simple and effective alternative, yielding faithful, sample-level uncertainty diagnostics with a tunable trade-off between accuracy and runtime.

\textbf{Contributions.} Our main contributions are:

\begin{itemize}[leftmargin=*,topsep=0pt,itemsep=1ex]
    \item \textbf{Fisher–Laplace projection.} We derive a closed-form projection of parameter uncertainty into data space via the Jacobian of the denoiser, yielding an interpretable epistemic uncertainty map that separates parameter uncertainty from diffusion noise.


    \item \textbf{FLARE.} We introduce a scalable randomized approximation that subsamples parameters uniformly across the network and prove that it preserves epistemic structure with rapidly decaying relative error.

    \item \textbf{Experiments.} We demonstrate improved uncertainty-aware sample filtering on synthetic time-series tasks, achieving up to $100\%$ gap closure and consistently outperforming BayesDiff and last-layer Laplace baselines.
\end{itemize}

\subsection{Preliminaries on Diffusion Models}

We consider a discrete-time score-based diffusion model~\citep{nichol2021improved, ho2020ddpm} with time steps $t\in\{1,\dots,T\}$ and data space $\mathbb{R}^d$.
Let $\vx_0 \sim p_{\mathrm{data}}$ and define the forward diffusion process
\begin{equation}
q(\vx_t \mid \vx_{t-1}) \;=\; \mathcal{N}\!\big(\sqrt{\alpha_t}\,\vx_{t-1},\; (1-\alpha_t)\mI\big),
\end{equation}
with $\bar\alpha_t = \prod_{s=1}^t \alpha_s$, so that
\begin{equation}
q(\vx_t \mid \vx_0)=\mathcal{N}\!\big(\sqrt{\bar\alpha_t}\,\vx_0,\,(1-\bar\alpha_t)\mI\big).
\end{equation}
The reverse dynamics are parameterized by a neural network $\varepsilon_\theta(\vx_t,t)$ with parameters $\theta\in\mathbb{R}^p$ and follow the standard DDPM update~\citep{ho2020ddpm}
\begin{equation}
\begin{aligned}
  \vx_{t-1} &= a_t \vx_t - b_t\,\varepsilon_\theta(\vx_t,t) + \veta_t,\\
  \veta_t &\sim \mathcal{N}(\bm0,\tilde\beta_t \mI),
\end{aligned}
\label{eq:ddpm-update}
\end{equation}
where $a_t=\alpha_t^{-1/2}$,  
$b_t={\beta_t}/\big(\sqrt{\alpha_t}\sqrt{1-\bar\alpha_t}\big)$, and  
$\tilde\beta_t=\tfrac{1-\alpha_{t-1}}{1-\bar\alpha_t}\beta_t$ (with $\alpha_0=1$).
Throughout, we adopt the $\varepsilon$-prediction parameterization.

The denoiser $\varepsilon_\theta$ is trained via score matching on pairs $(\vx_0,t)$ sampled from the data distribution and a time prior, by minimizing a weighted mean-squared error between the predicted and true noise \citep{hang2023efficient}.
At sampling time, the diffusion schedule is fixed and randomness is introduced through $\veta_t$, which represents \emph{sampling noise} independent of the model parameters $\theta$.

\subsection{Formalizing the UQ Problem}

We define predictive uncertainty for a single reverse step as the conditional covariance
\[
\Cov(\vx_{t-1}\mid \vx_t, t).
\]
By the law of total covariance,
\begin{equation}\label{eq:prelim-ltv}
\begin{aligned}
\Cov(\vx_{t-1}\mid \vx_t, t)
&= \underbrace{\E_{\theta}\!\big[\Cov(\vx_{t-1}\mid \vx_t, t, \theta)\big]}_{\text{aleatoric}}\\
&\quad+\; \underbrace{\Cov_{\theta}\!\big(\E[\vx_{t-1}\mid \vx_t, t, \theta]\big)}_{\text{epistemic}},
\end{aligned}
\end{equation}
where expectations and covariances over $\theta$ are taken with respect to a parameter distribution, such as the posterior $p(\theta\mid\mathcal D)$ given training data $\mathcal D$.
Here, $\theta$ is treated as a random variable encoding uncertainty over which model generated the data.
The second term measures how predictions vary across plausible models (epistemic uncertainty), while the first term captures the intrinsic randomness of the reverse step given a fixed model (aleatoric uncertainty).
Together, they separate uncertainty due to limited knowledge from irreducible stochasticity in the generative process.

For the DDPM update, the conditional moments are
\begin{align}
\E[\vx_{t-1}\mid \vx_t,t,\theta] &= a_t \vx_t - b_t\,\varepsilon_\theta(\vx_t,t), \\
\Cov(\vx_{t-1}\mid \vx_t,t,\theta) &= \tilde\beta_t \mI.
\end{align}
The aleatoric component in \cref{eq:prelim-ltv} therefore corresponds directly to the diffusion noise $\tilde\beta_t \mI$, while the epistemic component arises solely from uncertainty in the denoiser parameters $\theta$.
Unlike aleatoric uncertainty, epistemic uncertainty is \emph{reducible}: it decreases as more data are observed or as the model class better matches the true generative process. \cref{eq:prelim-ltv} thus characterizes the one-step predictive variability of the reverse diffusion.
However, in practice, uncertainty is often assessed at the level of entire trajectories or in the variability of the final sample $\vx_0$.
Such quantities can be obtained by propagating and accumulating the one-step uncertainties along the reverse chain, yielding a pathwise notion of uncertainty that remains consistent with the  diffusion dynamics.

Throughout, we assume a fixed diffusion schedule, Gaussian forward noise, and a differentiable denoiser $\varepsilon_\theta$.
Unless otherwise stated, conditioning on $(\vx_t, t)$ removes randomness from the data distribution and all variables at time steps greater than \(t\), isolating the stochasticity of the reverse transition from uncertainty in $\theta$.
Because the denoiser is applied repeatedly, small modeling errors can compound along the reverse trajectory.
As a result, epistemic uncertainty is most pronounced in low-density regions, between modes, or under distribution shift, where the model must extrapolate beyond its training data.
Aleatoric variability, by contrast, reflects only the inherent randomness of sampling.
Conflating the two, such as by interpreting sample variance as total uncertainty, can thus obscure whether variability reflects genuine model uncertainty or stochastic noise.

\section{Related Work}

UQ is an active area in deep learning, motivated by the need to assess model reliability and robustness.
A standard distinction is between aleatoric uncertainty, which arises from inherent randomness in the data, and epistemic uncertainty, which reflects uncertainty about the model parameters~\citep{kendall2017uncertainties,hullermeier2021aleatoric}.

Several approaches have been proposed to capture epistemic uncertainty in neural networks.
Variational Bayesian methods introduce probabilistic structure but often underestimate uncertainty in practice.
Ensemble methods~\citep{lakshminarayanan2017deep} typically provide stronger calibration but are expensive to train.
Dropout-based approximations~\citep{gal2016dropout} offer computational efficiency but tend to blur the separation between aleatoric and epistemic effects.
More recently, Laplace-based methods~\citep{daxberger2021laplace,maddox2019simple} have emerged as scalable second-order approximations that trade some accuracy for improved efficiency.

These strategies have also been applied to generative models, including variational autoencoders (VAEs)~\citep{hirschfeld2020uncertainty,bohm2019uncertainty}, generative adversarial networks (GANs)~\citep{oberdiek2022uqgan,salimans2016improved}, and normalizing flows~\citep{papamakarios2021normalizing}.
While such methods provide useful uncertainty signals, the high-dimensional parameterization of modern generative models makes exact Bayesian inference infeasible, and practical approximations often conflate epistemic and aleatoric sources of variability.

Diffusion models introduce additional challenges, as uncertainty accumulates across multiple denoising steps.
BayesDiff~\citep{kou2023bayesdiff} applies a last-layer Laplace approximation and propagates variance through the reverse process, but this approximation collapses much of the parameter-induced structure and mixes epistemic and aleatoric contributions.
Hypernetwork-based approaches such as HyperDiffusion~\citep{chan2024estimating} introduce parameter variability via an auxiliary network, capturing richer effects at the cost of increased model complexity and training overhead.
Ensemble-based methods~\citep{berry2024casting,jazbec2025generative} quantify uncertainty through model or parameter disagreement but require multiple forward passes per sample.
Parameter perturbation approaches~\citep{shu2024zero} provide a more direct handle on epistemic variability, albeit with increased computational cost.

\section{Method}
\label{sec:method}

In this section, we introduce a Fisher information-guided method that isolates epistemic (parameter) uncertainty by projecting a Laplace posterior over model parameters into data space at each diffusion reverse step. We then propagate and accumulate these one-step contributions along the denoising trajectory via a simple linear recursion induced by the DDPM update. This recursion tracks parameter sensitivity along the reverse path, yielding a final epistemic covariance at $\vx_0$ and enabling sample-level diagnostics such as trace-based scores and variance maps.
Finally, to make the approach scalable, we introduce a randomized subnetwork estimator that subsamples columns of the generalized Gauss-Newton (GGN) matrix (i.e., the Fisher matrix under a squared-loss objective). This estimator preserves the global structure of the network at a fraction of the computational cost, while avoiding the inaccuracies associated with methods such as LLLA.

\paragraph{Aleatoric randomness and conditioning.}
We let $\eta$ denote the aggregate non-parametric randomness in the system, including
(i) training-side randomness (e.g., data sampling and stochastic optimization noise), and
(ii) sampling-side randomness (e.g., forward diffusion noise used to define $\vx_t$ from $\vx_0$, latent trajectory randomness, and, in DDPMs, explicit reverse-step noise).
For a fixed realization of $\eta$, the reverse trajectory $\{\vx_t(\theta,\eta)\}_{t=0}^T$ is a deterministic function of the model parameters $\theta$.

\subsection{Fisher Information-guided Projection}

Let $\hat\theta$ denote the maximum a posteriori (MAP) estimate of the model parameters given training data $\mathcal D$.
We approximate the posterior distribution locally using a second-order (Laplace) expansion.

The posterior covariance over parameters is
\begin{equation}
\mSigma_\theta \equiv \Cov[\vtheta \mid \mathcal D]
= \E\!\big[(\vtheta - \hat{\vtheta})(\vtheta - \hat{\vtheta})^\top \mid \mathcal D\big].
\end{equation}
Under a local Laplace approximation around $\hat\theta$, this covariance takes the form
$\mSigma_\theta \approx (\mH + \lambda \mI)^{-1}$,
where $\mH$ is the GGN matrix of the training loss and $\lambda$ is a small damping parameter.
The Jacobian of the denoiser with respect to the model parameters is
\begin{equation}
\mJ_t \coloneqq
\nabla_\theta \varepsilon_\theta(\vx_t,t)\big|_{\hat\theta}
\in \mathbb{R}^{d\times p}.
\label{eq:Jt}
\end{equation}
All Jacobians $\mJ_t$ are evaluated along the MAP sampling trajectory
$\{\vx_t(\hat\theta,\eta)\}_{t=0}^T$ corresponding to a fixed realization of $\eta$.
Throughout, we adopt the $\varepsilon$-prediction parameterization; alternative parameterizations modify only the scalar coefficients $(a_t,b_t)$.

We define the epistemic covariance of the reverse state $\vx_t$ as
\begin{equation}
\mSigma^{\mathrm{ep}}_t(\eta)
\coloneqq
\Cov_\theta\!\big(\vx_t(\theta,\eta)\mid \eta\big).
\end{equation}
Intuitively, this quantity captures how uncertainty in the model parameters propagates into uncertainty in the reverse diffusion state, holding all non-parametric randomness fixed.
\Cref{eq:one-step-fisher-method} describes how to compute its contribution over a single reverse step.

\begin{mybox}{One-step Fisher–Laplace Projection}
\setlength{\abovedisplayskip}{4pt}
\setlength{\belowdisplayskip}{4pt}
\setlength{\abovedisplayshortskip}{4pt}
\setlength{\belowdisplayshortskip}{4pt}
\begin{equation}
\label{eq:one-step-fisher-method}
\mSigma^{\mathrm{ep}}_{t-1\mid t}(\eta)
=
b_t^2\,\mJ_t\,\mSigma_\theta\,\mJ_t^\top .
\end{equation}
\end{mybox}

The validity of the one-step projection in \Cref{eq:one-step-fisher-method}
is discussed in \Cref{sec:theory_one_step}.

\subsection{Propagation through the Denoising Steps}

We now describe the central mechanism of our method: a recursion that propagates epistemic uncertainty along the reverse diffusion trajectory.
Conditioning on a fixed realization of the non-parametric randomness $\eta$, we propagate epistemic covariance backward in time using a linear update induced by the DDPM dynamics.
\cref{eq:ep-recursion} accumulates the effect of parameter uncertainty across successive denoising steps.
\begin{mybox}{\textcolor{black}{Multi-step Epistemic Recursion}}
\setlength{\abovedisplayskip}{4pt}
\setlength{\belowdisplayskip}{4pt}
\setlength{\abovedisplayshortskip}{4pt}
\setlength{\belowdisplayshortskip}{4pt}
\begin{equation}
\label{eq:ep-recursion}
\mSigma^{\mathrm{ep}}_{t-1}(\eta)
=
a_t^2\,\mSigma^{\mathrm{ep}}_t(\eta)
+
b_t^2\,\mJ_t\,\mSigma_\theta\,\mJ_t^\top .
\end{equation}
\end{mybox}

The recursion in \Cref{eq:ep-recursion} is justified and analyzed in \Cref{sec:theory_recursion}.
Unrolling it yields the closed-form expression in~\cref{eq:ep-closed-form}.

\begin{mybox}{\textcolor{black}{Unrolled Epistemic Accumulation}}
\setlength{\abovedisplayskip}{4pt}
\setlength{\belowdisplayskip}{4pt}
\setlength{\abovedisplayshortskip}{4pt}
\setlength{\belowdisplayshortskip}{4pt}
\begin{equation}
\label{eq:ep-closed-form}
\mSigma^{\mathrm{ep}}_{0}(\eta)
=\sum_{s=1}^{T}\Bigg(\prod_{j< s} a_j\Bigg)^{\!2}
\,b_s^2\; \mJ_s\,\mSigma_\theta\,\mJ_s^\top.
\end{equation}
\end{mybox}

While \Cref{eq:ep-closed-form} provides a clear characterization of epistemic uncertainty, directly evaluating this expression requires projecting a full-parameter Laplace posterior through the network at every reverse step.

\paragraph{Computational considerations.}
The full Fisher--Laplace projection
$b_t^2 \mJ_t \mSigma_\theta \mJ_t^\top$
provides the most faithful estimate of epistemic uncertainty, as it accounts for parameter sensitivity throughout the entire network.
However, forming, storing, or inverting the full posterior covariance
$\mSigma_\theta \in \mathbb{R}^{p \times p}$
is computationally prohibitive in large-scale models, where the number of parameters $p$ is large.
This motivates practical approximations that trade accuracy for scalability.

A common simplification \citep{daxberger2021laplace} restricts the Laplace posterior to the \emph{final affine layer} of the network.
For example, \citet{kou2023bayesdiff} adopt this approach to quantify uncertainty in diffusion models.
Algebraically, this replaces $\mSigma_\theta$ with a covariance
$\mSigma_{\theta,\mathrm{last}}$ defined only over the final-layer parameters, and projects uncertainty using the corresponding subset of columns of $\mJ_t$.
While computationally efficient, this approximation discards parameter sensitivity originating in earlier layers and propagated through the network.
As we show in \cref{sec:exp}, it can underestimate epistemic structure along the reverse trajectory, particularly when uncertainty arises from deep feature representations rather than the final linear head.
These limitations motivate an alternative that remains computationally tractable while retaining sensitivity contributions from across the network.

\subsection{Randomized Subnetwork Approximation}
\label{sec:practical_estimators}

We propose a scalable alternative that preserves network-wide parameter sensitivity while substantially reducing computational cost.
Our approach approximates the Fisher--Laplace projection by operating on a randomly selected subnetwork of parameters at each reverse step.
Rather than restricting uncertainty to a fixed subset of parameters, the method randomly subsamples columns of both the Jacobian and the GGN matrix.
This preserves sensitivity contributions from across the network while avoiding the cost of full Fisher inversion.
The procedure is given by:

\begin{tcolorbox}[left=1mm,right=1mm]
\begin{enumerate}
    \item Sample a uniform index set $I \subset \{1,\ldots,p\}$ with cardinality $|I| = m \ll p$.
    
    \item Restrict the GGN matrix to the selected coordinates, forming
    $\mH_{I,I} \in \mathbb{R}^{m \times m}$, and define the subnetwork posterior covariance
    $$\mSigma_{\mathrm{sub}} = (\mH_{I,I} + \lambda \mI)^{-1}.$$
    
    \item Replace the full Jacobian by its selected columns
    $\mJ_{t,I} \in \mathbb{R}^{d \times m}$ and use this subnetwork within the recursion of \cref{eq:ep-recursion}:
    \begin{equation*}
      \mSigma^{\mathrm{ep}}_{t-1}
      \;=\; a_t^2\,\mSigma^{\mathrm{ep}}_{t}
      \;+\; b_t^2\,\mJ_{t,I}\,\mSigma_{\mathrm{sub}}\,\mJ_{t,I}^\top.
    \end{equation*}
\end{enumerate}
\end{tcolorbox}

This randomized approximation preserves the structure of the original Fisher--Laplace projection while reducing both memory and computational requirements.
Unlike LLLA, it captures epistemic contributions originating in intermediate representations and propagated through the network.
\cref{alg:Algorithm1} summarizes the resulting Fisher--Laplace Randomized Estimator (FLARE), which integrates the randomized algorithm with our recursion formula to produce both a sample $\hat{\vx}_0$ and its associated epistemic covariance.

\paragraph{Theoretical guarantees.}
In \Cref{sec:theory_randomized}, we analyze the randomized subnetwork estimator using tools from randomized numerical linear algebra~\citep{murray2023randomized}.
Under mild regularity assumptions, we show that the approximation error of the randomized projection decays as $\mathcal{O}(1/\sqrt{m})$, yielding a principled trade-off between computational cost and estimation accuracy.
By contrast, last-layer Laplace approximations do not admit comparable guarantees under the same assumptions, as they restrict uncertainty to a fixed subset of parameters corresponding to the final layer and thus discard network-wide sensitivity.

\paragraph{Approximating the Hessian.}
Our method requires access to curvature information of the denoiser network $\varepsilon_\theta(\vx_t,t)$ in the form of a GGN matrix. For the mean-squared error objective used to train the denoiser, the GGN matrix at the MAP estimate satisfies
\begin{equation}
  \mH \;\approx\; \frac{1}{n}\sum_{i=1}^n \mJ_i^\top \mJ_i,
  \qquad
  \mJ_i \coloneqq \nabla_\theta \varepsilon_\theta(\vx_i,t_i)\big|_{\hat\theta},
\end{equation}
where $(\vx_i,t_i)$ are training pairs.
Equivalently, stacking the Jacobians row-wise to form the population Jacobian
$\mJ_{\mathrm{pop}} \in \mathbb{R}^{(nd)\times p}$ yields
$\mH \approx \tfrac{1}{n}\mJ_{\mathrm{pop}}^\top \mJ_{\mathrm{pop}}$.

Restricting $\mJ_{\mathrm{pop}}$ to the final-layer coordinates recovers the standard LLLA.
In contrast, our randomized subnetwork approach uniformly subsamples $m$ columns of $\mJ_{\mathrm{pop}}$, yielding a reduced GGN that retains sensitivity contributions from across the network.
This reduces both compute and memory costs from scaling with $p$ to scaling with $m$, while avoiding the structural bias introduced by last-layer restriction.

To disentangle the effects of full-curvature modeling from epistemic--aleatoric separation, we include controlled full-Hessian ablations and four-way uncertainty comparisons in Appendix~\ref{app:cross-term}, Fig.~\ref{fig:fourway_uq}.

\begin{algorithm}[t]
\caption{FLARE}
\label{alg:Algorithm1}
\footnotesize

\SetKwInOut{Input}{Inputs}
\SetKwInOut{Output}{Outputs}

\Input{
diffusion model $\varepsilon_\theta$;
subnetwork size $m$;
Hessian--vector products for $(\mH + \lambda \mI)$;
DDPM schedule $\{\beta_t\}_{t=1}^T$;
$\vx_T \sim \mathcal{N}(0,\mI)$.
}

\Output{
generated sample $\hat{\vx}_0$;
epistemic covariance $\mSigma^{\mathrm{ep}}(\hat{\vx}_0)$.
}

\BlankLine
\textbf{Precompute:}
$\alpha_t \gets 1 - \beta_t$\;
$\bar{\alpha}_t \gets \prod_{s \le t} \alpha_s$\;
$a_t \gets \alpha_t^{-1/2}$\;
$b_t \gets \beta_t / (\sqrt{\alpha_t}\sqrt{1-\bar{\alpha}_t})$.

\BlankLine
\textbf{Sample subnetwork (once):}
draw $I \subset [p]$\;
define solve operator $\mSigma_{\mathrm{sub}} v \coloneqq (\mH_{I,I} + \lambda \mI)^{-1} v$

\BlankLine
\textbf{Initialize:}
$\hat{\vx}_T \gets \vx_T$,\;
$\mSigma^{\mathrm{ep}}(\hat{\vx}_T) \gets \bm{0}$\;

\BlankLine
\For{$t = T$ \KwTo $1$}{
  $\hat{\varepsilon}_t \gets \varepsilon_\theta(\hat{\vx}_t, t)$\;
  $\hat{\vx}_{t-1} \gets a_t\!\left(\hat{\vx}_t - \frac{\beta_t}{\sqrt{1-\bar{\alpha}_t}}\,\hat{\varepsilon}_t\right)$\;
  $\mJ_{t,I} \gets \nabla_{\theta_I}\varepsilon_\theta(\hat{\vx}_t, t)$\;
  $\Delta_t \gets \mJ_{t,I}\,\mSigma_{\mathrm{sub}}\,\mJ_{t,I}^\top$\;
  $\mSigma^{\mathrm{ep}}(\hat{\vx}_{t-1}) \gets a_t^2\,\mSigma^{\mathrm{ep}}(\hat{\vx}_t) + b_t^2\,\Delta_t$\;
}

\BlankLine
\Return{$\hat{\vx}_0,\ \mSigma^{\mathrm{ep}}(\hat{\vx}_0)$}
\end{algorithm}

\subsection{Diagnostics}

\cref{alg:Algorithm1} produces, for each generated sample $\hat{\vx}_0$, an associated epistemic covariance $\mSigma^{\mathrm{ep}}_0(\eta)$.
We use this covariance to define scalar diagnostics that quantify epistemic uncertainty at the sample level.

Our primary metric is the trace $\tr(\mSigma_0^{\rm ep}(\eta))$, which equals the sum of per-dimension epistemic variances.
When reporting an average variance, we use the normalized trace $\tr(\mSigma_0^{\rm ep}(\eta)) / d$.
For cross-dataset comparisons, we apply additional normalization to ensure scale invariance.
In all experiments, we rank samples by these scores and retain those with the lowest $\tr(\mSigma_0^{\rm ep}(\eta))$ as our most confidently generated samples.

In practice, we often seek scalar diagnostics without explicitly forming the full covariance matrix.
For example, computing $\tr(\mSigma^{\mathrm{ep}}_0(\eta))$ requires evaluating
$\tr(\mSigma^{\mathrm{ep}}_{t-1\mid t}(\eta))$ at each reverse step along the denoising trajectory.
Let $\vg_{t,k} \in \mathbb{R}^p$ denote the gradient of the $k$-th output of $\varepsilon_\theta(\vx_t, t)$, i.e., the $k$-th row of $\mJ_t$.
Then,
\[
    \tr(\mSigma_{t-1\mid t}^{\rm ep}(\eta)) = b_t^2 \sum_{k=1}^d u_{t,k},
\]
where $u_{t,k} \coloneqq \vg_{t,k}^\top \mSigma_\theta \vg_{t,k}$.
Each $u_{t,k}$ can be computed efficiently using the conjugate gradient method.
Specifically, since $\mSigma_\theta \approx (\mH + \lambda \mI)^{-1}$, we solve
\begin{equation}
(\mH + \lambda \mI)\, \vz = \vg_{t,k},
\end{equation}
where $\vz \in \mathbb{R}^p$ denotes the solution to the corresponding linear system,
using Hessian--vector products (e.g., the empirical Fisher),
evaluate $u_{t,k} = \vg_{t,k}^\top \vz$, and sum across dimensions.
This avoids explicit matrix inversion and is standard in Laplace-based approximations.

\subsection{Relation to BayesDiff}
\label{sec:bayesdiff}

BayesDiff~\citep{kou2023bayesdiff} estimates \emph{pixel-wise predictive uncertainty}
by placing a Laplace posterior on the \emph{noise predictions} of a diffusion model
and propagating data-space variances through the reverse diffusion process.
Its recursion tracks the total variance of the latent state, combining diffusion noise,
output uncertainty of the denoiser, and a data-space cross-covariance
$\Cov(x_t,\varepsilon_t)$ estimated via Monte Carlo sampling.
In contrast, our method operates directly in \emph{parameter space}.
We define epistemic uncertainty as the covariance of the reverse trajectory
induced solely by uncertainty in the model parameters, conditioning on a fixed
realization of the diffusion randomness.
This leads to a recursion that propagates parameter-induced variability through the reverse dynamics via
\[
\Cov_\theta\!\big(\E[x_{t-1}\mid x_t,\theta]\big)
= b_t^2\, J_t\, \Sigma_\theta\, J_t^\top,
\]
a quantity that is neither estimated nor approximated in BayesDiff. As a result, the two approaches quantify fundamentally different notions of uncertainty.
BayesDiff targets total predictive uncertainty in data space, aggregating multiple sources of variability,
whereas our method isolates epistemic uncertainty arising from parameter uncertainty and propagates it explicitly through the reverse diffusion process.

\section{Theory}
\label{sec:theory}

Section~\ref{sec:method} introduced two computational primitives:
(i) the \emph{one-step Fisher--Laplace projection} in \Cref{eq:one-step-fisher-method}, and
(ii) the \emph{multi-step epistemic propagation} recursion in \Cref{eq:ep-recursion}.
This section provides theoretical justification for both constructions, as well as for the
randomized subnetwork approximation introduced in \Cref{sec:practical_estimators}.
All proofs are deferred to the appendix, as indicated throughout.

\subsection{Epistemic Uncertainty and Propagation}
\label{sec:theory_recursion}

We justify the recursion in \Cref{eq:ep-recursion} used to propagate epistemic covariance
along the reverse diffusion trajectory.
We proceed in three steps:
first, we establish the one-step Fisher--Laplace projection term
(\Cref{eq:one-step-fisher-method});
second, we derive a general propagation identity that includes a cross-covariance term;
and finally, we state a local decoupling condition under which this cross term vanishes,
yielding the recursion used in our method.

\paragraph{One-step Fisher--Laplace projection.}
\label{sec:theory_one_step}

Our analysis follows from the law of total variance together with the DDPM reverse update
\citep{ho2020ddpm}.
Let $\hat\theta$ denote the maximum a posteriori (MAP) estimate of the model parameters,
around which the posterior distribution is locally approximated using a second-order
(Laplace) expansion~\citep{daxberger2021laplace}.

\begin{tcolorbox}[left=1mm,right=1mm]
\begin{proposition}[One-step Fisher--Laplace projection]
\label{prop:one-step-fisher}
Under a local Gaussian posterior $\theta \sim \mathcal{N}(\hat\theta,\mSigma_\theta)$,
independence of the reverse-step noise $\eta_t$ from the model parameters $\theta$,
and a first-order (delta-method) linearization of the one-step conditional mean around
$\hat\theta$, the conditional covariance at step $t$ admits the decomposition
\begin{equation}
\begin{aligned}
\Cov_\theta\!\big(
\vx_{t-1} \mid \vx_t, t, \eta
\big)
&=
\underbrace{\tilde\beta_t \mI}_{\text{aleatoric}}
\\
&\quad
+
\underbrace{
b_t^{2}\,\mJ_t\,\mSigma_\theta\,\mJ_t^\top
}_{\text{epistemic}}
+
o(\lVert\mSigma_\theta\rVert).
\end{aligned}
\end{equation}
\end{proposition}
\end{tcolorbox}

Consequently, the epistemic contribution to the reverse-step uncertainty, which we isolate and propagate in our method, is given by
\begin{align}
\mSigma^{\mathrm{ep}}_{t-1\mid t}(\eta)
&\coloneqq
\Cov_\theta\!\big(
\E[\vx_{t-1} \mid \vx_t, t, \theta, \eta]
\;\big|\;\eta
\big) \nonumber \\
&\approx
b_t^{2}\,\mJ_t\,\mSigma_\theta\,\mJ_t^\top .
\label{eqn:conditional-covar}
\end{align}
This result shows that epistemic uncertainty can be expressed explicitly in terms of
parameter sensitivity, without conflating parameter uncertainty with diffusion noise. 
The above result holds under mild regularity conditions, formalized in
\Cref{app:prop-proof}.
In particular, we assume that the denoiser $\varepsilon_\theta(\cdot,t)$ is continuously
differentiable in $\theta$ with a locally Lipschitz Jacobian, that the diffusion schedule
is fixed, and that the injected noise $\eta_t$ (a component of $\eta$) is independent of the model parameters.
The posterior covariance $\Cov[\theta \mid \mathcal D] = \mSigma_\theta$ is assumed to have
finite second moments.
Under a Laplace approximation, we take
\[
\mSigma_\theta \approx (\mH + \lambda \mI)^{-1},
\]
where $\mH$ denotes the Gauss--Newton approximation to the Hessian of the weighted
mean-squared error (MSE) training loss~\citep{ritter2018scalable}.

\paragraph{Propagation through the denoising steps.}

\cref{eqn:conditional-covar} characterizes the epistemic uncertainty contributed by a single reverse step.
To understand how this uncertainty accumulates along the reverse trajectory, we define
\[
\mSigma^{\mathrm{ep}}_t(\eta)
\coloneqq
\Cov_\theta\!\big(\vx_t(\theta,\eta)\mid \eta\big),
\]
the epistemic covariance of the reverse state at time $t$, conditioned on a fixed realization of the non-parametric randomness $\eta$.

Under the $\varepsilon$-prediction DDPM update and a first-order linearization around the MAP estimate, the epistemic covariance evolves according to
\begin{equation}
\begin{aligned}
\mSigma^{\mathrm{ep}}_{t-1}(\eta)
&=
a_t^2\,\mSigma^{\mathrm{ep}}_{t}(\eta)
+
b_t^2\,\mJ_t\,\mSigma_\theta\,\mJ_t^\top
\\
&\quad
+
2a_t b_t\,\mC_t(\eta)
+
o(\lVert\mSigma_\theta\rVert),
\end{aligned}
\label{eq:ep-recursion-with-cross}
\end{equation}
where the cross-covariance term is defined as
\[
\mC_t(\eta)
\coloneqq
\Cov_\theta\!\big(\vx_t(\theta,\eta),\,\mJ_t(\theta-\hat\theta)\mid\eta\big).
\]

To obtain a tractable recursion, we consider the regime in which the parameter-induced perturbation of the reverse state
$\vx_t(\theta,\eta)$ is approximately uncorrelated with the linearized denoiser response evaluated at the MAP iterate.
Conditioned on a given realization of the non-parametric randomness $\eta$, this corresponds to the approximation
\begin{equation}
\Cov_\theta\!\Big(
\vx_t(\theta,\eta),\;
\mJ_t(\theta-\hat\theta)
\;\Big|\;\eta
\Big)
\;\approx\;
\mathbf{0}.
\label{eq:decoupling}
\end{equation}

Under this local decoupling approximation, the cross-covariance term in
\cref{eq:ep-recursion-with-cross} vanishes, yielding the additive recursion
in \Cref{eq:ep-recursion} used in our method.
Importantly, this approximation is not imposed arbitrarily.
In Appendix~\ref{app:cross-term}, we show that the cross term is
(i) uniformly bounded and absorbable into the leading variance terms,
(ii) asymptotically negligible under posterior concentration and local smoothness of the reverse trajectory, and
(iii) numerically insignificant under a full-Hessian Laplace posterior, as verified by Monte Carlo evaluation.

\subsection{Randomized Estimator}
\label{sec:theory_randomized}

We now provide theoretical support for the randomized subnetwork approximation introduced in
\Cref{sec:practical_estimators}.
Using tools from randomized numerical linear algebra~\citep{murray2023randomized,erichson2019randomized}, we show that the random subnetwork estimator
converges rapidly to the full Fisher--Laplace covariance under mild regularity conditions, while
last-layer restrictions generally do not admit comparable guarantees.
Our main result establishes a relative error bound in trace norm that decays with the subnetwork
size~$m$.

\paragraph{Setup.}
We analyze the approximation of the one-step epistemic covariance
$\mSigma^{\mathrm{ep}}_{t-1\mid t}$.
Under the generalized Gauss--Newton modeling assumption and
\Cref{eqn:conditional-covar}, this covariance can be written as
\begin{equation}
    \mSigma^{\mathrm{ep}}_{t-1\mid t}
    \;\approx\;
    \mJ_t\big(\mJ_{\mathrm{pop}}^\top\mJ_{\mathrm{pop}} + \lambda \mI\big)^{-1}\mJ_t^\top,
\end{equation}
where $\mJ_t$ denotes the Jacobian of the denoiser at step~$t$, and
$\mJ_{\mathrm{pop}}$ is the population Jacobian formed by stacking training-set Jacobians.

Since the regularization parameter $\lambda$ is typically very small
(e.g., $\lambda = 10^{-6}$ in our experiments),
we simplify the analysis by considering the limit $\lambda \to 0$.
This reduces the problem to approximating
\begin{equation}
    \mV \coloneqq \mJ_t\big(\mJ_{\mathrm{pop}}^\top\mJ_{\mathrm{pop}}\big)^{+}\mJ_t^\top,
\end{equation}
where $(\cdot)^{+}$ denotes the Moore--Penrose pseudoinverse.

Directly constructing $\mV$ is computationally infeasible in large models, as it requires access
to the full Jacobians.
Both the LLLA and the randomized subnetwork method approximate
$\mV$ by restricting attention to a subset of parameter coordinates.
Specifically, they form reduced Jacobians
$\tilde\mJ_t$ and $\tilde\mJ_{\mathrm{pop}}$ and compute
\begin{equation}
    \tilde\mV \coloneqq
    \tilde\mJ_t\big(\tilde\mJ_{\mathrm{pop}}^\top \tilde\mJ_{\mathrm{pop}}\big)^{+}\tilde\mJ_t^\top.
\end{equation}
The two approaches differ only in how these reduced Jacobians are chosen.
LLLA deterministically restricts to the coordinates corresponding to the final affine layer,
whereas the randomized subnetwork method samples a uniformly random subset of $m$ coordinates.

\paragraph{Least-squares reformulation.}
The following observation, proved in \Cref{app:theory}, allows us to analyze these approximations
through an equivalent least-squares perspective.

\begin{lemma}
\label{lem:cov-to-ls}
Let $\mA \in \mathbb{R}^{p \times k}$ and $\mB \in \mathbb{R}^{p \times \ell}$ be full-rank matrices
with $k,\ell < p$.
Let $\mX_\star \in \mathbb{R}^{k \times \ell}$ denote the solution to
$\min_{\mX}\|\mA\mX - \mB\|_{\rm F}$.
Then $\mB^\top(\mA\mA^\top)^+\mB = \mX_\star^\top\mX_\star$.
\end{lemma}

Applying this lemma yields the representations
\begin{alignat*}{3}
    \mV &= \mX_\star^\top\mX_\star,
    \quad
    \mX_\star \coloneqq \argmin_{\mX}\|\mJ_{\mathrm{pop}}^\top\mX - \mJ_t^\top\|_{\rm F}^2, \\
    \tilde\mV &= \tilde\mX^\top\tilde\mX,
    \quad
    \tilde\mX \coloneqq \argmin_{\mX}\|\tilde\mJ_{\mathrm{pop}}^\top\mX - \tilde\mJ_t^\top\|_{\rm F}^2.
\end{alignat*}
Consequently, $\tilde\mV$ is a good approximation of $\mV$ if and only if the reduced least-squares
problem provides an accurate approximation of the full one.

From this perspective, the LLLA implicitly assumes that the least-squares
problem induced by $\mJ_{\mathrm{pop}}^\top$ and $\mJ_t^\top$ can be well approximated using only the
coordinates corresponding to the final layer.
This assumption is generally not guaranteed to hold for overdetermined least-squares problems and
helps explain the empirical limitations of LLLA observed in practice.

\paragraph{Randomized subnetwork analysis.}
In contrast, the randomized subnetwork approximation corresponds to solving the same least-squares
problem using a uniformly random subset of rows.
This strategy is well studied in randomized numerical linear algebra and is known to yield reliable
approximations under mild conditions.
Specifically, we assume:
\begin{itemize}
    \item \textbf{Conditioning.} The population Jacobian matrix $\mJ_{\mathrm{pop}}^\top$ has condition number $\kappa$.
    \item \textbf{Coherence.} The population Jacobian has coherence $\mu$, ruling out concentration
    of mass on a small subset of coordinates.
    \item \textbf{Alignment.} The step-wise Jacobian $\mJ_t$ has nontrivial overlap with the rowspace
    of $\mJ_{\mathrm{pop}}$, quantified by a constant $\gamma > 0$.
\end{itemize}

Under these conditions, we obtain the following trace-norm approximation guarantee.

\begin{tcolorbox}[left=1mm,right=1mm]
\begin{theorem}
\label{thm:random-subnetwork}
Let $\mJ_{\mathrm{pop}}\in\mathbb{R}^{nd \times p}$ and $\mJ_t\in\mathbb{R}^{d \times p}$ satisfy the
assumptions above.
Then, with probability at least $1-\delta$, the randomized subnetwork estimator with subnetwork size
$m$ satisfies
\begin{equation}
    \big|\tr(\mV - \tilde\mV)\big|
    \;\leq\;
    \tilde{\mathcal O}\!\left(
    \sqrt{\frac{p\,\gamma\,\mu}{m\,\delta}}\,\kappa
    \right)\tr(\mV).
\end{equation}
\end{theorem}
\end{tcolorbox}

As $m$ increases, the approximation error decays to zero.
By a standard union bound, this guarantee can be extended to hold uniformly over the $T$ reverse
steps used in \Cref{alg:Algorithm1}.

\section{Experimental Results}
\label{sec:exp}

In this section, we evaluate the effectiveness of our proposed method (\Cref{alg:Algorithm1}) for quantifying epistemic uncertainty in samples generated by diffusion models.
We consider three synthetic time-series benchmarks designed to probe complementary aspects of epistemic uncertainty:
(i) bimodal sinusoids, (ii) chirped sinusoids, and (iii) damped sinusoids.
These tasks test multi-modality, ambiguity, extrapolation, and temporal decay.
We compare our approach against established baselines and evaluate both qualitative behavior and quantitative performance using standardized metrics for sample quality and confidence-based filtering.

Detailed descriptions of the datasets, diffusion models, and parameterizations are provided in \Cref{app:datasets} and \Cref{app:hyper}.

\textbf{Baselines and evaluation protocol.}
We compare against two baseline methods:
(i) BayesDiff predictive variance~\citep{kou2023bayesdiff}, and
(ii) last-layer Laplace (LLLA) rollouts~\citep{daxberger2021laplace}.
All methods are evaluated under matched training and sampling protocols. Each method assigns a scalar uncertainty score $u$ to every generated sample. For a given method, we form a filtered subset by retaining the lowest-uncertainty $p\%$ of samples, with $p=50$ for the sine and grid datasets and $p=25$ for the chirp dataset (the latter reduces visual clutter given the longer sequence length $L=80$).
Unfiltered generations, i.e., using all samples, serve as the baseline.

\textbf{Quantitative evaluation.}
To assess the effect of uncertainty-based filtering, we train a discriminator $D_\phi$ to distinguish training samples from generated samples, labeling training data as $1$ and generated data as $0$.
Its accuracy is
$
\mathrm{Acc} \;=\; \frac{1}{N}\sum_{i=1}^N \mathbbm{1}\{D_\phi(x_i)=y_i\},
$
where $y_i \in \{0,1\}$ indicates if $x_i$ is drawn from the training set.
To quantify the impact of filtering, we report \emph{gap closure}, defined as
\[
\mathrm{Gap\text{-}Closure}
\;=\;
\frac{\bigl|0.5 - \mathrm{Acc}_f\bigr| - \bigl|0.5 - \mathrm{Acc}_{u\!f}\bigr|}
     {\bigl|0.5 - \mathrm{Acc}_{u\!f}\bigr|},
\]
where $\mathrm{Acc}_f$ denotes the discriminator accuracy on the filtered samples and
$\mathrm{Acc}_{u\!f}$ the accuracy on the unfiltered baseline.
Subtracting from \(0.5\) measures by how much we are outperforming random chance,
with smaller values indicating less distinguishability between generated and training samples.
Gap closure measures how much filtering reduces the discriminator’s ability to distinguish generated samples \textit{relative to the unfiltered baseline}.
A gap-closure of $100\%$ is good, indicating that the discriminator can no longer distinguish filtered samples from training data.
Values between $0\%$ and $100\%$ indicate partial improvement, while negative values indicate that filtering actually removes high quality samples and degrades realism.
We also compute the Receiver Operating Characteristic Area Under the Curve (ROC-AUC),
defined as
$
\mathrm{AUC} \;=\; \Pr\!\bigl(s(x^+) > s(x^-)\bigr),
$
where $s(\cdot)$ denotes the discriminator score, $x^+$ is a training sample, and $x^-$ is a generated sample.
An AUC of $0.5$ corresponds to chance-level discrimination (good), while larger values indicate greater separability (bad).
Statistical significance is assessed using bootstrap resampling. For each method, we estimate the distribution of gap-closure values under resampling and report $p$-values, indicating whether improvements are consistent.

\begin{table}[!b]
\caption{Comparison of uncertainty-based filtering methods across two datasets.
Higher gap-closure and lower ROC-AUC indicate improved alignment between filtered generations and the training data.
Best values are highlighted in bold.}

\label{tab:gap_closure_results}
\label{Table 1}
\centering
\scriptsize
\setlength{\tabcolsep}{4pt}
\renewcommand{\arraystretch}{1.2}
\resizebox{\columnwidth}{!}{%
\begin{tabular}{l l c c c}
\toprule
Dataset & Method & \makecell{Gap-Closure (\%)} & \makecell{ROC-AUC\\($\downarrow$)} & \makecell{$p$ (bootstrap)} \\
\midrule
\multirow{3}{*}{Sines}
& BayesDiff     & +13.3654   & 0.6153 & 0.0001 \\
& LLLA         & +47.1873   & 0.5814 & 0.0201 \\
& FLARE (ours)  & \textbf{+93.0814}   & \textbf{0.5003} & \textbf{0.0012} \\
\midrule
\multirow{3}{*}{Chirp}
& BayesDiff     & +41.7337   & 0.6616 & 0.0030 \\
& LLLA         & +59.0583  & 0.5891 & 0.0116 \\
& FLARE (ours)  & \textbf{+74.3137}  & \textbf{0.5345} & \textbf{0.0002} \\
\midrule
\multirow{3}{*}{Damped Sines}
& BayesDiff     & +10.3984   & 0.6861 & 0.0001 \\
& LLLA         & $-18.7540$  & 0.7754 & 0.0066 \\
& FLARE (ours)  & \textbf{+85.0}  & \textbf{0.5085} & \textbf{0.0002} \\
\bottomrule
\end{tabular}%
}
\end{table}

\textbf{Results.}
As shown in \Cref{tab:gap_closure_results}, BayesDiff predictive variance yields modest improvements, with $+13.4\%$ gap closure on the sine dataset and $+41.7\%$ on the chirp dataset, both with statistical significance.
The last-layer Laplace approximation (LLLA) often underperforms, exhibiting negative or marginal gap closures and relatively high ROC-AUC values.
In contrast, our method consistently achieves the strongest gains, with $+93.0\%$ gap closure on the sine dataset, $+74.3\%$ on the chirp dataset, and $+76.2\%$ on the damped sine dataset, alongside the lowest ROC-AUC values.
These results indicate that filtering with FLARE substantially reduces the discriminator’s ability to distinguish generated samples from training data across all three settings. All improvements are statistically significant under bootstrap testing.
Overall, the quantitative results align with qualitative inspection, indicating that FLARE provides a more informative epistemic uncertainty signal than existing baselines. \emph{Notably, the largest gains occur in settings that require extrapolation or mode selection, where separating epistemic from aleatoric uncertainty is most critical.}

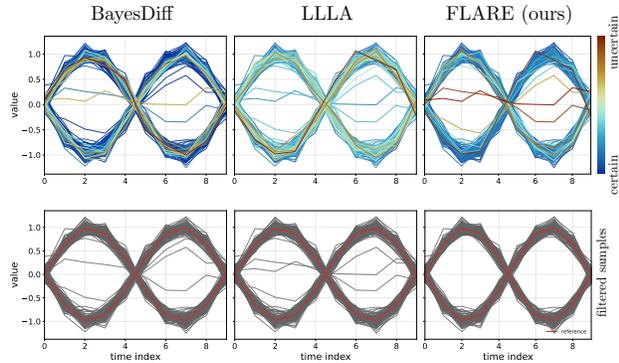
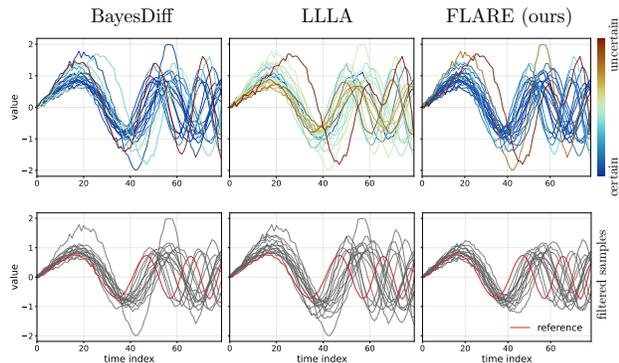
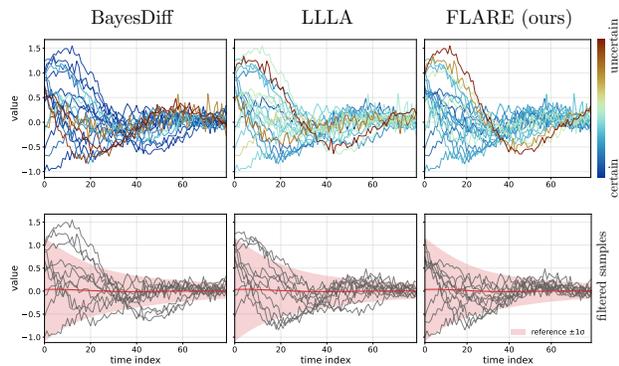
\begin{figure}[!t]
    \centering

    \vspace{+0.4cm}
    \begin{subfigure}{\linewidth}
        \centering
        \begin{overpic}[width=0.97\linewidth]
            {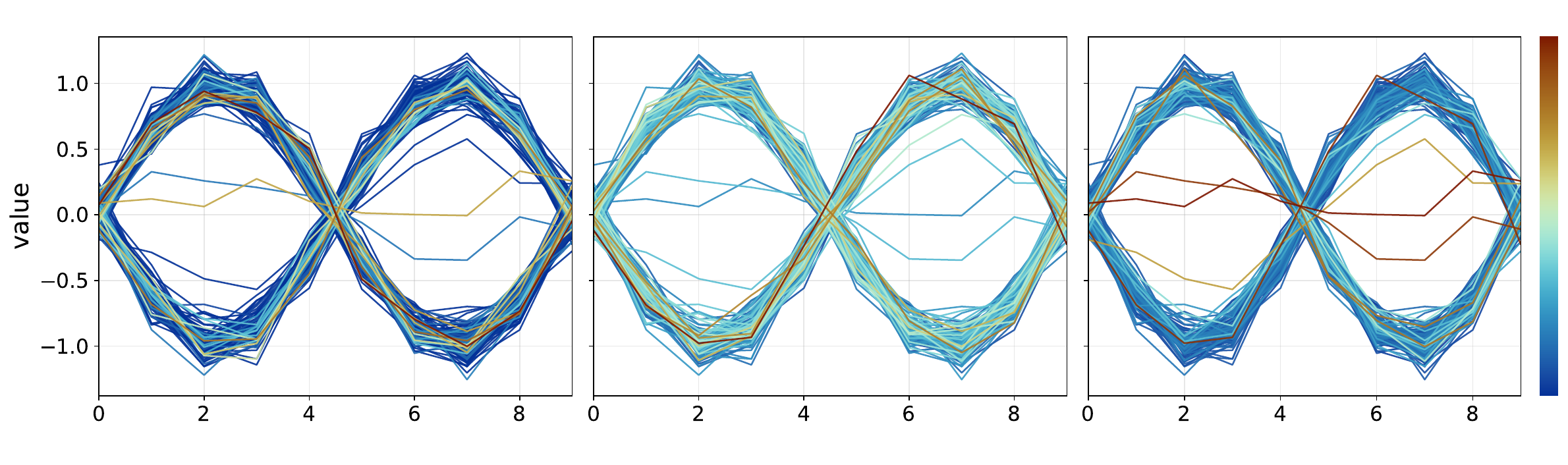}
            \put(14,30){{\scalebox{0.75}{BayesDiff}}}
            \put(49,30){{\scalebox{0.75}{LLLA}}}
            \put(73,30){{\scalebox{0.75}{FLARE (ours)}}}
            \put(100,1){\rotatebox{90}{\scalebox{0.5}{certain}}}
            \put(100,21){\rotatebox{90}{\scalebox{0.5}{ uncertain}}}
        \end{overpic}
        
        \begin{overpic}[width=0.97\linewidth]
            {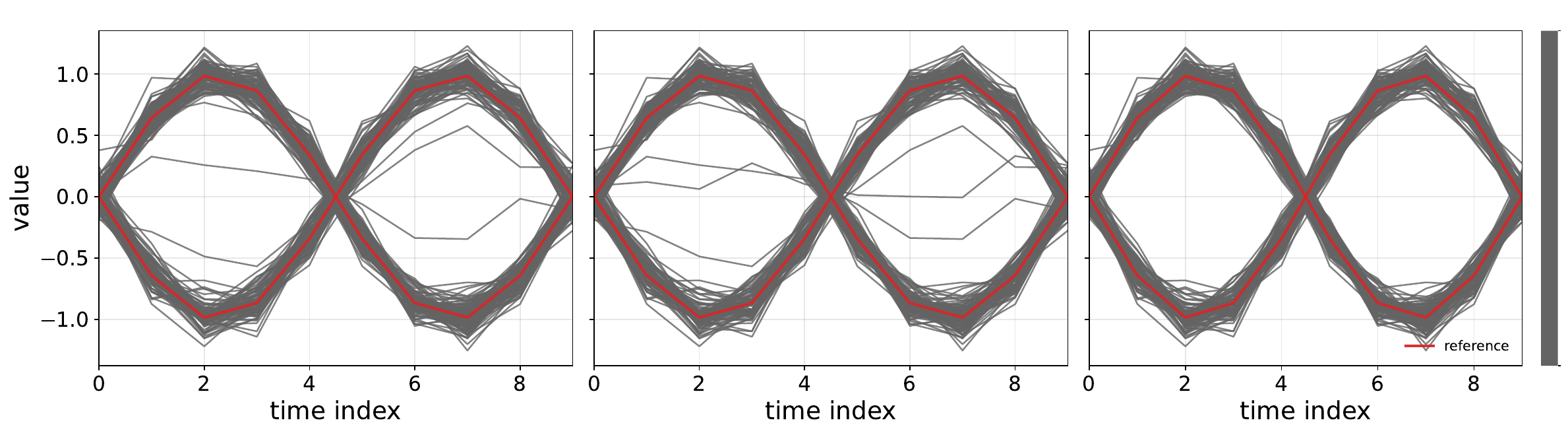}
        \put(98,3){\color{white}\rule{1em}{5.6em}} 
        \put(98,8){\rotatebox{90}{\scalebox{0.5}{filtered samples}}}
        \end{overpic}
        \subcaption{Sinusoidal time series dataset.}
        \label{fig:sines}
    \end{subfigure}

    \vspace{0.7cm} 
    \begin{subfigure}{\linewidth}
        \centering
        \begin{overpic}[width=0.97\linewidth]
            {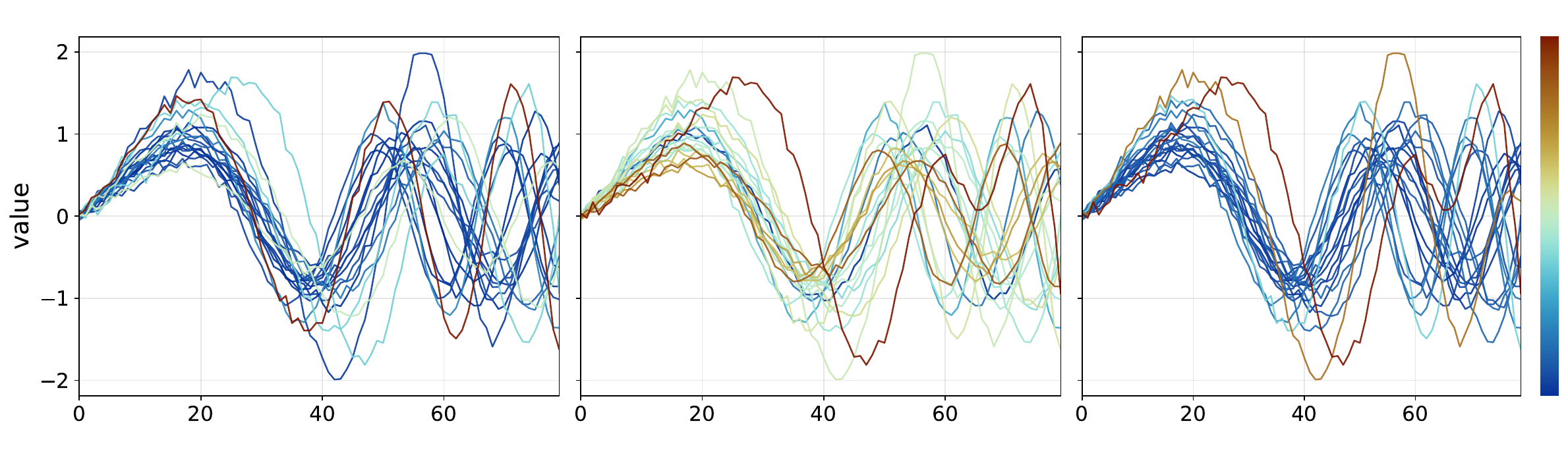}
            \put(14,30){{\scalebox{0.75}{BayesDiff}}}
            \put(49,30){{\scalebox{0.75}{LLLA}}}
            \put(73,30){{\scalebox{0.75}{FLARE (ours)}}}
            \put(100,1){\rotatebox{90}{\scalebox{0.5}{certain}}}
            \put(100,21){\rotatebox{90}{\scalebox{0.5}{ uncertain}}}
        \end{overpic}
        \begin{overpic}[width=0.97\linewidth]
            {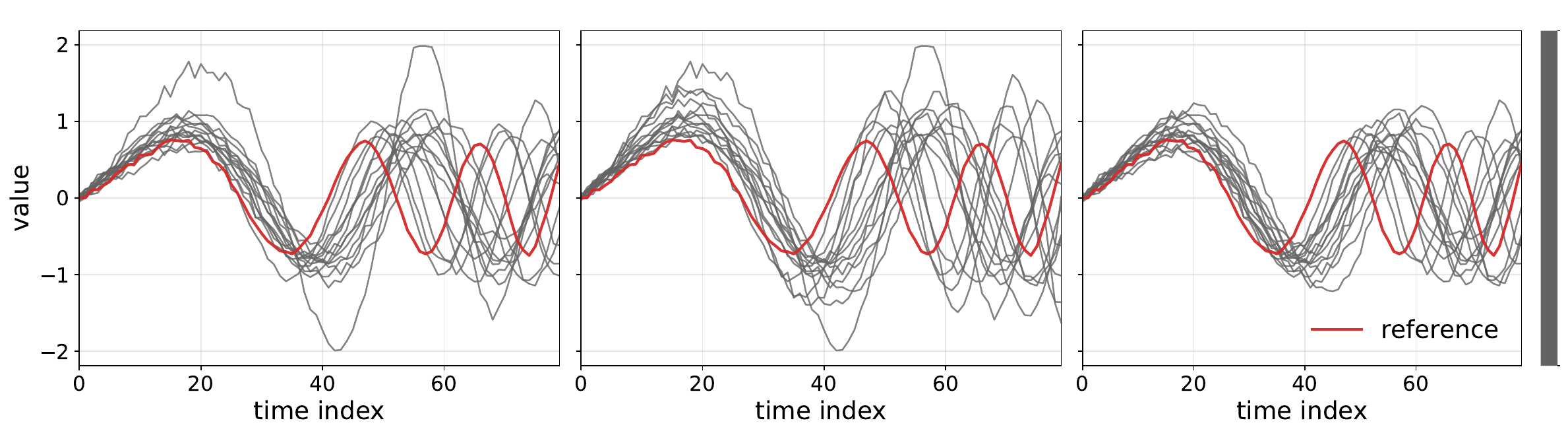}
            \put(98,3){\color{white}\rule{1em}{5.6em}} 
            \put(98,8){\rotatebox{90}{\scalebox{0.5}{filtered samples}}}
        \end{overpic}
        \subcaption{Chirp time series dataset.}
        \label{fig:chirp}
    \end{subfigure}

    \vspace{0.7cm} 
    \begin{subfigure}{\linewidth}
        \centering
        \begin{overpic}[width=0.97\linewidth]
        {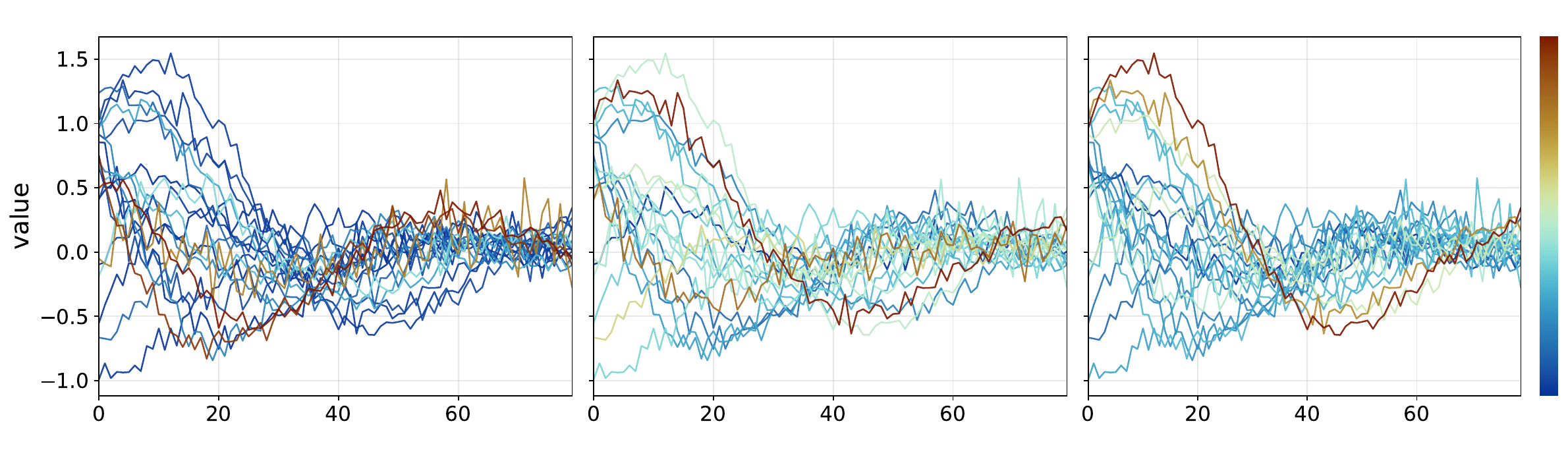}
            \put(14,30){{\scalebox{0.75}{BayesDiff}}}
            \put(49,30){{\scalebox{0.75}{LLLA}}}
            \put(73,30){{\scalebox{0.75}{FLARE (ours)}}}
            \put(100,1){\rotatebox{90}{\scalebox{0.5}{certain}}}
            \put(100,21){\rotatebox{90}{\scalebox{0.5}{ uncertain}}}
         \end{overpic}
        \begin{overpic}[width=0.97\linewidth]
        {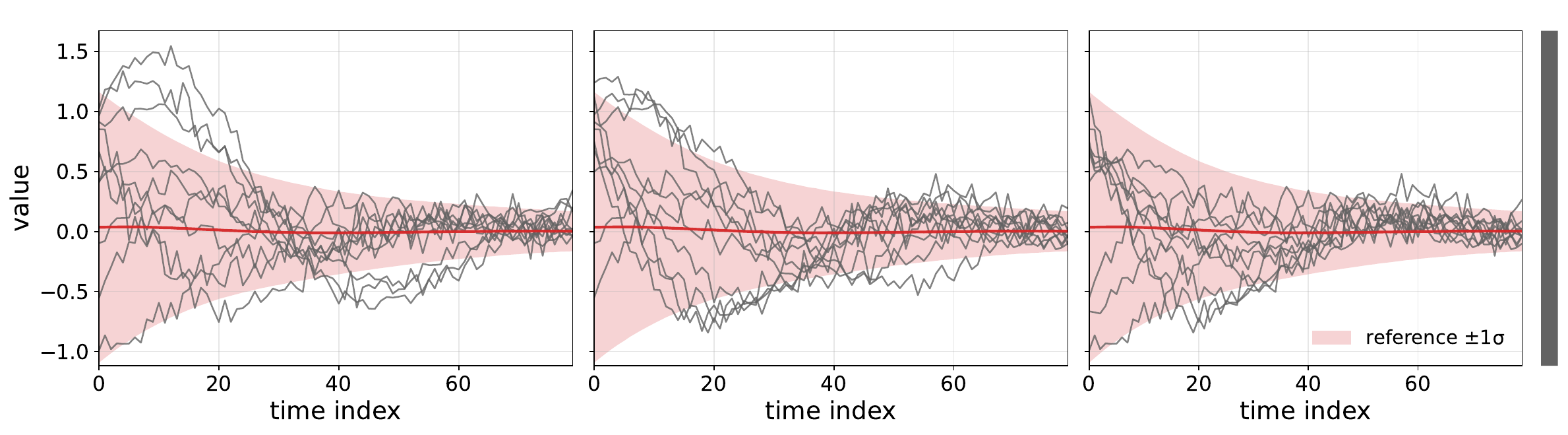}
        \put(98,3){\color{white}\rule{1em}{5.6em}} 
        \put(98,8){\rotatebox{90}{\scalebox{0.5}{filtered samples}}}
         \end{overpic}
        \subcaption{Damped Sinusoidal time series dataset.}
        \label{fig:damped}
    \end{subfigure}
    
\caption{Generated time-series samples before and after uncertainty-based filtering.
Panel~(a) shows the sinusoidal dataset and panel~(b) the chirp dataset.
\textbf{Top:} generated trajectories, colored by epistemic uncertainty (blue = low, red = high).
\textbf{Bottom:} trajectories retained after filtering by uncertainty.
Filtering removes implausible, off-manifold samples while preserving diverse, on-distribution trajectories.}

    \label{fig:timeseries}
\end{figure}

\textbf{Qualitative evaluation.}
Uncertainty-based filtering also improves the visual fidelity of generated samples.
In the 2D grid dataset (\Cref{fig:grid}), filtered samples concentrate tightly around the mixture components, with reduced spillover into low-density regions.
On the sinusoidal time-series benchmark (\Cref{fig:sines}), low-uncertainty subsets cluster around the two sinusoidal modes, while samples in ambiguous regions between modes are pruned.
For both the chirped and damped sinusoid datasets (\Cref{fig:chirp}, \Cref{fig:damped}), filtering suppresses spurious oscillations and frequency drift, yielding smoother trajectories that better match the training distribution.
\emph{These visual patterns mirror the quantitative findings and provide evidence that FLARE identifies uncertainty in regions where the model must interpolate, extrapolate, or resolve ambiguity.}

\section{Conclusion and Limitations}

Quantifying epistemic uncertainty in diffusion models is challenging because uncertainty arises from multiple sources and accumulates across many denoising steps.
Existing approaches either conflate epistemic uncertainty with diffusion noise or restrict uncertainty estimation to a small subset of parameters, thereby discarding important structure.
In this work, we argue that reliable uncertainty quantification in diffusion models requires explicitly tracking how parameter uncertainty propagates through the reverse dynamics.

To this end, we introduced a Fisher–Laplace formulation that projects parameter uncertainty into data space via the Jacobian of the denoiser and then accumulates this uncertainty along a realized reverse trajectory.
Building on this formulation, we proposed FLARE, a randomized subnetwork estimator that makes Fisher–Laplace uncertainty estimation scalable to modern diffusion models.
FLARE preserves network-wide sensitivity while substantially reducing computational cost, and we provide both theoretical guarantees and empirical evidence that it yields more informative epistemic uncertainty estimates than last-layer or predictive-variance-based alternatives.

Empirically, we demonstrated that FLARE produces structured, sample-level uncertainty signals that align with intuitive notions of confidence, particularly in regions of interpolation, extrapolation, and low data density.
These uncertainty estimates enable reliable filtering of generated samples, improving sample quality without retraining, ensembling, or repeated inference-time sampling.
More broadly, our results suggest that separating epistemic and aleatoric effects is essential for interpreting uncertainty in diffusion-based generative models.

\textbf{Limitations.} We evaluated FLARE primarily on synthetic time-series benchmarks designed to probe specific epistemic failure modes; extending the approach to higher-dimensional domains such as multivariate time series or image generation remains an important direction for future work.
In addition, while randomized subnetwork sampling significantly reduces computational overhead, approximating curvature information in very large diffusion models is still demanding.
Future work could explore structured or adaptive subsampling strategies, low-rank curvature approximations, or tighter integrations with efficient second-order optimization methods to further scale FLARE.

%% file: appendix.tex
\clearpage
\appendix

\onecolumn

\section{Proof of \texorpdfstring{\Cref{prop:one-step-fisher}}{Proposition \ref*{prop:one-step-fisher}}}
\label{app:prop-proof}


In this section, we prove \Cref{prop:one-step-fisher}.
We start by restating the setup and formalizing notation further.
Let $d$ be the data dimension and $p$ the number of parameters.
We adopt the $\varepsilon$-prediction reverse update
\begin{equation}
\label{eq:reverse-update}
\mathbf{x}_{t-1}
=
a_t\,\mathbf{x}_t
-
b_t\,\varepsilon_{\boldsymbol{\theta}}(\mathbf{x}_t,t)
+
\boldsymbol{\eta}_t,
\qquad
\boldsymbol{\eta}_t\sim\mathcal N(\mathbf{0},\tilde\beta_t\mathbf{I}_d).
\end{equation}
where $a_t,b_t,\tilde\sigma_t>0$ are schedule constants.
We take $\hat{\boldsymbol{\theta}}$ to be the maximum a posteriori (MAP) estimate of the denoiser’s parameters $\boldsymbol{\theta}$ given training data $D$ and prior $p(\boldsymbol{\theta})$:
\[
\hat{\boldsymbol{\theta}}
~\coloneqq~
\arg\max_{\boldsymbol{\theta}} \log p(\boldsymbol{\theta}\mid D)
~=~
\arg\max_{\boldsymbol{\theta}}
\left\{ \log p(D\mid \boldsymbol{\theta}) + \log p(\boldsymbol{\theta}) \right\}.
\]

We define the parameter--Jacobian of the denoiser at the current sampler state as
\[
\mathbf{J}_t
~\coloneqq~
\nabla_{\boldsymbol{\theta}}\,\varepsilon_{\boldsymbol{\theta}}(\mathbf{x}_t,t)\big|_{\hat{\boldsymbol{\theta}}}
\in\mathbb{R}^{d\times p},
\]
which depends on $t$ because it is evaluated at the realized pair $(\mathbf{x}_t,t)$ along the reverse trajectory.

In order to prove \cref{prop:one-step-fisher}, we will make a mild assumption about the smoothness of our model.
\begin{assumption}
    \label{assumption:assumptions}
            For all time steps \(t\), we assume the denoiser \(\eps_{\vtheta}(\cdot,t)\) is continuously differentiable in a neighborhood around the MAP estimate \(\hat\vtheta\) that includes \(\vtheta\).
\end{assumption}

We can now prove the main result:

\begin{reptheorem}{prop:one-step-fisher}
Under a local Gaussian posterior \(\vtheta \sim \mathcal N(\hat\vtheta,\mSigma_{\vtheta})\), independence of the schedule noise \(\veta_t\) from \(\vtheta\), and \cref{assumption:assumptions}, the conditional covariance at reverse step \(t\) can be decomposed as
\[
\mathrm{Cov}(\mathbf{x}_{t-1}\mid \mathbf{x}_t,t)
=
\tilde\beta_t\,\mathbf{I}_d
\;+\;
b_t^2\,\mathbf{J}_t\,\boldsymbol{\Sigma}_\theta\,\mathbf{J}_t^\top
\;+\;
o(\|\boldsymbol{\Sigma}_\theta\|).
\]
Consequently, the epistemic contribution to the covariance is
\[
\boldsymbol{\Sigma}^{\mathrm{ep}}_{t-1\mid t}
~\coloneqq~
\mathrm{Cov}_{\boldsymbol{\theta}}\!\big(\,\mathbb{E}[\mathbf{x}_{t-1}\mid \mathbf{x}_t,t,\boldsymbol{\theta}]\,\big)
~=~
b_t^2\,\mathbf{J}_t\,\boldsymbol{\Sigma}_\theta\,\mathbf{J}_t^\top
\;+\;
o(\|\boldsymbol{\Sigma}_\theta\|).
\]
\end{reptheorem}

\begin{proof}
Recall \cref{eq:prelim-ltv}, which uses the law of total variance to say
\[
\mathrm{Cov}(\mathbf{x}_{t-1} \mid \mathbf{x}_t, t)
=
\mathbb{E}_{\boldsymbol{\theta}}\!\big[\, \mathrm{Cov}(\mathbf{x}_{t-1} \mid \mathbf{x}_t, t, \boldsymbol{\theta}) \,\big]
\;+\;
\mathrm{Cov}_{\boldsymbol{\theta}}\!\big( \, \mathbb{E}[\mathbf{x}_{t-1} \mid \mathbf{x}_t, t, \boldsymbol{\theta}] \, \big).
\]
We will analyze these two terms separately.


Recall the reverse update rule under \(\eps\)-prediction, stated in \cref{eq:reverse-update}.
Conditioned on \((\vx_t,t)\) and a fixed parameter vector \(\vtheta\), the only randomness in the update rule originates from the injection noise \(\veta_t\), and so
\begin{equation}\label{eq:aleatoric}
    \E_{\vtheta}[\Cov(\vx_{t-1} ~|~ \vx_t, t, \vtheta)
    = \E_{\vtheta}[\Cov(\veta_t)]
    = \E_{\vtheta}[\tilde\beta_t \mI_d]
    = \tilde\beta_t \mI_d.
\end{equation}
That is, the aleatoric portion of the covariance of \(\vx_{t-1}\) given \((\vx_t,t)\) is exactly \(\tilde\beta_t \mI_d\).

We now turn to the second term from the law of total variance.
From the reverse update rule \cref{eq:reverse-update}, we have
\(
    \E[\vx_{t-1} ~|~ \vx_t,t,\vtheta]
    = a_t \vx_t - b_t \eps_{\hat\vtheta}(\vx_t,t).
\)
By \cref{assumption:assumptions}, we linearly approximate \(\eps_{\vtheta}\) around the MAP estimate \(\hat\vtheta\):
\begin{equation}\label{eq:approx}
    \eps_{\vtheta}(\vx_t,t)
    = \eps_{\hat\vtheta}(\vx_t,t) + \mJ_t (\vtheta - \hat\vtheta) + \vr_t(\theta)
\end{equation}
where \(\mJ_t = \nabla_{\vtheta}\eps_{\vtheta}(\vx_t,t)|_{\hat\theta}\) as defined earlier and the remainder function satisfies \(\|\vr_t(\vtheta)\| = O(\|\vtheta - \hat\vtheta\|^2)\).
Substituting this into the conditional mean, we can define
\begin{equation}
\label{eq:ct-mu}
\mathbf{c}_t := a_t \mathbf{x}_t - b_t\,\varepsilon_{\hat{\boldsymbol{\theta}}}(\mathbf{x}_t, t),
\qquad
\text{so that}\quad
\boldsymbol{\mu}_{\boldsymbol{\theta}}
:= \mathbb{E}[\mathbf{x}_{t-1} \mid \mathbf{x}_t, t, \boldsymbol{\theta}]
= \mathbf{c}_t - b_t \mathbf{J}_t (\boldsymbol{\theta} - \hat{\boldsymbol{\theta}}) - b_t \mathbf{r}_t(\boldsymbol{\theta}).
\end{equation}
We now examine the covariance of all three terms in \(\mu_{\vtheta}\) with respect to \(\vtheta\sim\mathcal N (\hat\vtheta ,\mSigma_{\vtheta})\).
Since \(\vc_t\) is independent of \(\vtheta\), we know \(\Cov_{\vtheta}(\vc_t) = \bm 0\).
We also directly compute \(\Cov(b_t \mJ_t (\vtheta - \hat\vtheta)) = b_t^2\mJ_t\mSigma_{\vtheta}\mJ_t^\top\).

To tackle the third term, recall that the norm of \(\vr_t(\vtheta)\) is quadratic in \(\|\vtheta-\hat\vtheta\|\).
So, standard delta-method bounds imply that \(\|\Cov_{\vtheta}(\vr_t(\vtheta))\| = o(\|\mSigma_{\vtheta}\|_2)\) and that the cross-covariance between \(\vr_t(\vtheta)\) and \(b_t^2\mJ_t(\vtheta-\hat\vtheta)\) is \(o(\|\mSigma_{\vtheta}\|_2)\).
Therefore,
\begin{equation}\label{eq:epistemic-term}
\mathrm{Cov}_{\boldsymbol{\theta}}\!\big( \mathbb{E}[\mathbf{x}_{t-1} \mid \mathbf{x}_t, t, \boldsymbol{\theta}] \big)
=
b_t^{2}\, \mathbf{J}_t \boldsymbol{\Sigma}_\theta \mathbf{J}_t^\top \;+\; o(\|\boldsymbol{\Sigma}_\theta\|_{2}),
\end{equation}
which is the epistemic contribution.
Combining \cref{eq:aleatoric,eq:epistemic-term} yields
\[
\mathrm{Cov}(\mathbf{x}_{t-1} \mid \mathbf{x}_t, t)
=
\tilde{\beta}_t \mathbf{I}_d \;+\; b_t^{2}\, \mathbf{J}_t \boldsymbol{\Sigma}_\theta \mathbf{J}_t^\top \;+\; o(\|\boldsymbol{\Sigma}_\theta\|_{2}).
\]
We see that, the predictive covariance \(\Cov(\vx_{t-1} ~|~ \vx_t,t)\) decomposes into an aleatoric term from the injected noise and an epistemic term obtained by projecting the parameter posterior through the parameter--Jacobian at the current state.

\end{proof}

\paragraph{Explicit DDPM coefficients.}
For completeness, one may take \citep{ho2020ddpm,nichol2021improved}
\[
a_t=\frac{1}{\sqrt{\alpha_t}}\!\left(1-\frac{1-\alpha_t}{\sqrt{1-\bar\alpha_t}}\right),\quad
b_t=\frac{\sqrt{1-\alpha_t}}{\sqrt{\alpha_t}},\quad
\tilde\beta_t=\frac{1-\bar\alpha_{t-1}}{1-\bar\alpha_t}\,\beta_t,
\]
so that $m_{\vtheta}(\vx_t,t)=a_t \vx_t-b_t\varepsilon_{\vtheta}(\vx_t,t)$ and $\operatorname{Var}(\veta_t)=\tilde\beta_t \mI$.

\paragraph{Remarks.}
(i) The DDPM reverse update yields an \emph{exact} structural split
$\operatorname{Cov}(x_{t-1}\mid x_t,t)=\tilde\beta_t I+\operatorname{Cov}_\theta(m_\theta)$ before approximation; the ``aleatoric vs.\ epistemic'' terminology follows \citet{kendall2017uncertainties}. 
(ii) The epistemic term is a first-order pushforward of the parameter posterior:
$b_t^2 J_t\Sigma_\theta J_t^\top$, requiring only first-order network derivatives; the Laplace link between $\Sigma_\theta$ and (approximate) Fisher curvature is classical \citep{mackay1992practical,daxberger2021laplace}. 
(iii) For mean- or $x_0$-prediction, replace $b_t J_t$ by the Jacobian of the corresponding one-step mean $m_\theta$; the proof is unchanged. 
(iv) In the broader asymptotic/theoretical context for diffusion estimators and sampling, see \citet{chewi2025ddpm,chen2022sampling} for likelihood identities, efficiency guarantees, and sampling convergence results that justify using Fisher-type curvature as a proxy for epistemic uncertainty in DDPMs.
(v) The argument extends to DDIM~\citep{song2020denoising} and flow–matching/ODE samplers~\citep{lipman2024flow}: their one–step updates take the form $\mathbf{x}_{t-1}=m_{\boldsymbol{\theta}}(\mathbf{x}_t,t)$ (i.e., $\tilde\sigma_t=0$), so the aleatoric term $\tilde\beta_t\mathbf{I}_d$ disappears and the epistemic contribution reduces to the first–order pushforward through the mean map
\[
\mathrm{Cov}_{\boldsymbol{\theta}}\!\big(\mathbb{E}[\mathbf{x}_{t-1}\mid \mathbf{x}_t,t,\boldsymbol{\theta}]\big)
~=~ \big(\nabla_{\boldsymbol{\theta}} m_{\boldsymbol{\theta}}(\mathbf{x}_t,t)\big)\,
\boldsymbol{\Sigma}_\theta\,
\big(\nabla_{\boldsymbol{\theta}} m_{\boldsymbol{\theta}}(\mathbf{x}_t,t)\big)^\top
\;+\; o(\|\boldsymbol{\Sigma}_\theta\|_{2}).
\]

\section{Proofs for \texorpdfstring{\cref{sec:theory}}{Section \ref*{sec:theory}}}
\label{app:theory}

In this appendix, we prove \cref{lem:cov-to-ls}, which relates the covariance matrix construction to overdetermined least squares problems, and \cref{thm:random-subnetwork}, which uses this least squares perspective to prove that the random subnetwork approach

\begin{replemma}{lem:cov-to-ls}
    Let \(\mA\in\R^{p \times k}, \mB \in \R^{p \times \ell}\) be full rank with \(k,\ell < p\).
    Let \(\mX_\star \in \R^{k \times \ell}\) be the solution to the least squares problem \(\min_{\mX} \|\mA\mX-\mB\|_{\rm F}\).
    Then we have \(\mB^\top(\mA\mA^\top)^+\mB = \mX_\star^\top\mX_\star\).
\end{replemma}
\begin{proof}
    This follows directly from \((\mA\mA^\top)^+ = (\mA^+)^\top\mA^+\) and the fact that \(\mX_\star = \mA^+\mB\), since we can expand
    \[
        \mB^\top(\mA\mA^\top)^+\mB
        = (\mA^+\mB)^\top(\mA^+\mB)
        = \mX_\star^\top\mX_\star.
    \]
\end{proof}

We next turn our attention to \cref{thm:random-subnetwork}.
Before proving this result, we formally define the terms involved.
Let \(\mA \in \R^{p \times k}\) be full-rank matrix with \(p \geq k\) with economic SVD \(\mA = \mU\mSigma_\theta\mV^\top\) and singular values \(\sigma_1 \geq \ldots \geq \sigma_k\).
The \emph{condition number} of \(\mA\) is the ratio \(\kappa(\mA) \coloneq \frac{\sigma_1}{\sigma_k}\).
Since we expanded the economic SVD of \(\mA\), we have \(\mU \in \R^{p \times k}\) with orthonormal columns but not orthonormal rows.
The \emph{coherence} of \(\mA\) is the maximum squared row norm of \(\mU\).
That is, letting \(\vu_i^\top\) be the \(i^{th}\) row of \(\mU\), we have
\(
    \mu(\mA) \coloneq \max_{i \in \{1,\ldots,p\}} \|\vu_i\|_2^2 \in [\tfrac kp, 1  ]
\).

The result \cref{thm:random-subnetwork} follows from the well-understood facts about the sketch-and-solve algorithm from randomized numerical linear algebra \citep{woodruff2014sketching,Drineas17,meyer23}.
\begin{importedtheorem}
    \label{impthm:sketch-and-solve}
    Let \(\mA \in \R^{p \times k}\) be full-rank with \(p \geq k\) and let \(\mB \in \R^{p \times \ell}\).
    Assume that \(\|\mP\mB\|_{\rm F}^2 \geq \gamma \|(\mI-\mP)\mB\|_{\rm F}^2\) for some \(\gamma > 0\), where \(\mP \in \R^{p \times p}\) is the orthogonal projection onto the range of \(\mA\).
    Sample a uniformly random set \(I \subseteq \{1,\ldots,p\}\) with \(|I| = m\) for some \(m\) with \(k \leq m \leq p\).
    Let \(\tilde\mA \in \R^{m \times k}\) and \(\tilde\mB \in \R^{m \times \ell}\) be the restrictions of \(\mA\) and \(\mB\) to the rows indexed by \(I\).
    Define \(\mX_\star\) to be the solution of the full least squares problem
    \[
        \mX_\star \coloneq \argmin_{\mX\in\R^{k \times \ell}} \|\mA\mX-\mB\|_{\rm F}^2,
    \]
    and define \(\tilde\mX\) to be the solution of the subsampled problem
    \[
        \tilde\mX \coloneq \argmin_{\mX\in\R^{k \times \ell}} \|\tilde\mA\mX-\tilde\mB\|_{\rm F}^2.
    \]
    Then, with probability at least \(1-\delta\) we have that
    \[
        \|\mX_\star - \tilde\mX\|_{\rm F} \leq \tilde {\mathcal O}\left( \sqrt{\frac{p \gamma \mu(\mA)}{m \delta}} \kappa(\mA) \right) \|\mX_\star\|_{\rm F}.
    \]
\end{importedtheorem}
\begin{remark}
    This result does not appear exactly as stated above in the cited works.
    To recover this theorem statement, we first use the Corollary 1 from \citep{meyer23} to show that the uniformly random subsampling procedure produces an \emph{OSE Guarantee}.
    We then use the OSE Guarantee to activate Lemma 68 from \citep{Drineas17}.
    We further note the Lemma 68 is only stated for the case where \(\ell=1\).
    However, the proof of Lemma 68 goes through without change if we allow \(\ell>1\) and use the Frobenius norm in place of the vector \(\ell_2\) norm.
    These analyses are all standard, and have been presented in some form or another in many other works.
\end{remark}

From this result, we can now recover our main theorem statement.

\begin{reptheorem}{thm:random-subnetwork}
    Let \(\mJ_{\rm pop}\in\R^{nd \times p}\) and \(\mJ_t \in \R^{d \times p}\).
    Let \(\mu \in [\frac {nd}p, 1]\) and \(\kappa \geq 1\) be the \emph{coherence} and \emph{condition number} of \(\mJ_{\rm pop}^\top\), respectively.
    Let \(\mP\in\R^{p \times p}\) be the orthogonal projection onto the rowspace of \(\mJ_{\rm pop}\), and assume that \(\|\mJ_t\mP\|_{\rm F}^2\geq\gamma\|\mJ_t(\mI-\mP)\|_{\rm F}^2\) for some \(\gamma > 0\).
    Then, then random subnetwork approximation with subnetwork size \(m\) has
    \[
        |\tr(\mV - \tilde\mV)| \leq \tilde{\mathcal O}\bigg(\sqrt{\frac{p\gamma\mu}{m\delta}} \kappa\bigg)\tr(\mV)
    \]
    with probability at least \(1-\delta\).
\end{reptheorem}
\begin{proof}
    By \cref{lem:cov-to-ls}, we know that
    \begin{alignat*}{7}
        \mV &= \mX_\star^\top\mX_\star \quad&\text{where}\quad &&\mX_\star = \argmin_{\mX}\|\mJ_{\rm pop}^\top\mX - \mJ_t^\top\|_{\rm F}^2&\text{, and} \\
        \tilde\mV &= \tilde\mX^\top\tilde\mX \quad&\text{where}\quad &&\tilde\mX = \argmin_{\mX}\|\tilde\mJ_{\rm pop}^\top\mX - \tilde\mJ_t^\top\|_{\rm F}^2.
\end{alignat*}
    By \Cref{impthm:sketch-and-solve}, we know that \(\|\mX_\star - \tilde\mX\|_{\rm F} \leq \varepsilon \|\mX_\star\|_{\rm F}\) with probability at least \(1-\delta\), where \(\varepsilon = \tilde {\mathcal O}(\sqrt{\frac{p\gamma\mu}{m\delta}} \kappa)\).
    We thereby infer that \(|\|\mX_\star\|_{\rm F}^2 - \|\tilde\mX\|_{\rm F}^2| \leq 2\varepsilon \|\mX_\star\|_{\rm F}^2\) for small enough \(\varepsilon\).
    Recalling that \(\tr(\mA^\top\mA) = \|\mA\|_{\rm F}^2\), we observe that \(\tr(\mV) = \|\mX_\star\|_{\rm F}^2\) and therefore
    \begin{align*}
        |\tr(\mV - \tilde\mV)|
        = \big| \|\mX_\star\|_{\rm F}^2 - \|\tilde\mX\|_{\rm F}^2 \big|
        \leq 2\eps \|\mX_\star\|_{\rm F}^2
        = \tilde{\mathcal O}\left(\sqrt{\frac{p\gamma\mu}{m\delta}} \kappa\right) \tr(\mV)
    \end{align*}
\end{proof}

\section{Dataset Description}\label{app:datasets}

We used four simple synthetic data sets to test different aspects of uncertainty behavior. Each dataset is designed to expose a particular modeling challenge while remaining interpretable and visually intuitive. The collection includes both spatial and temporal domains, covering multimodal distributions as well as dynamic signals with varying frequency, phase, and decay. Together, these benchmarks provide controlled settings in which to study how uncertainty captures ambiguity, extrapolation, and signal degradation. They serve as minimal yet diverse examples that highlight where models should express high confidence and where epistemic uncertainty should naturally arise.

\begin{enumerate}
    \item \textbf{Grid (2D)}: a $3{\times}3$ modal Gaussian mixture with centers on $\{-1,0,1\}^2$ and per-mode covariance $0.05^2 I_2$, yielding high-support cells separated by low-density corridors. This setup highlights the need for spatially aware uncertainty, especially between modes where epistemic error should spike.

\item \textbf{Bimodal Sinusoidal time series} (length $L{=}10$): each sequence is drawn from a balanced mixture of $x_t=\pm \sin(2\pi \tau_t)+\varepsilon_t$, where $\tau_t$ is a uniform grid on $[0,1]$ and $\varepsilon_t\sim\mathcal{N}(0,\,0.1^2)$. This creates two separated modes with ambiguous regions in between.

\item \textbf{Chirp time series} (length $L{=}80$): sequences follow $x_t=A\sin\!\big(2\pi(f_0 \tau_t+\tfrac{1}{2}k \tau_t^2)\big)+\varepsilon_t$, with $A\sim\mathcal{U}[0.6,1.4]$, $f_0\sim\mathcal{U}[0.5,1.0]$, $k\sim\mathcal{U}[2.0,5.0]$, $\tau_t$ an $80$-point grid on $[0,1]$, and $\varepsilon_t\sim\mathcal{N}(0,\,0.02^2)$. The time-varying frequency stresses extrapolation and off-manifold behavior.

\item \textbf{Damped sinusoidal time series} (length $L{=}40$): sequences decay smoothly in amplitude, following $x_t = A e^{-d t} \sin\big(2\pi f t + \phi\big) + \varepsilon_t$, with $A$, $f$, $d$, and $\phi$ drawn from uniform ranges. This benchmark evaluates how well uncertainty tracks vanishing signal energy and long-horizon extrapolation under structured decay.
\end{enumerate}

\begin{figure}[!t]
    \centering
    \begin{minipage}[t]{0.32\linewidth}
        \centering
        \includegraphics[width=\linewidth]{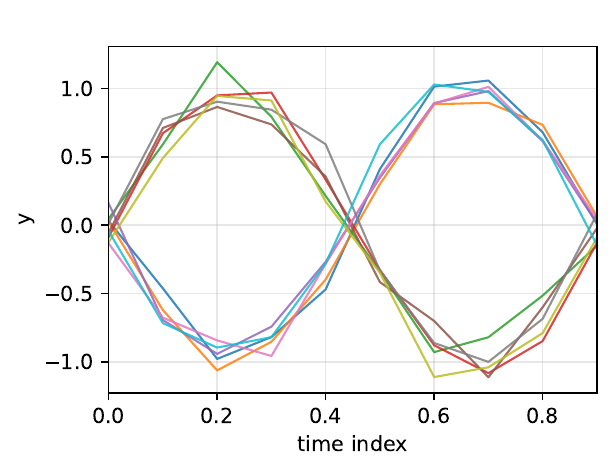}
        \caption*{\textbf{(a)} Bimodal Sinusoidal}
    \end{minipage}
    \hfill
    \begin{minipage}[t]{0.32\linewidth}
        \centering
        \includegraphics[width=\linewidth]{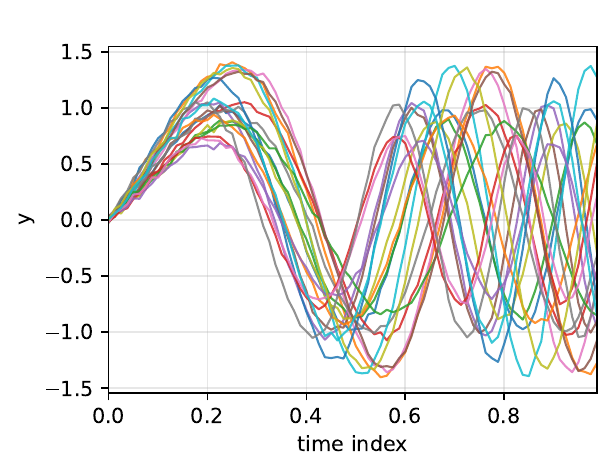}
        \caption*{\textbf{(b)} Chirp}
    \end{minipage}
    \hfill
    \begin{minipage}[t]{0.32\linewidth}
        \centering
        \includegraphics[width=\linewidth]{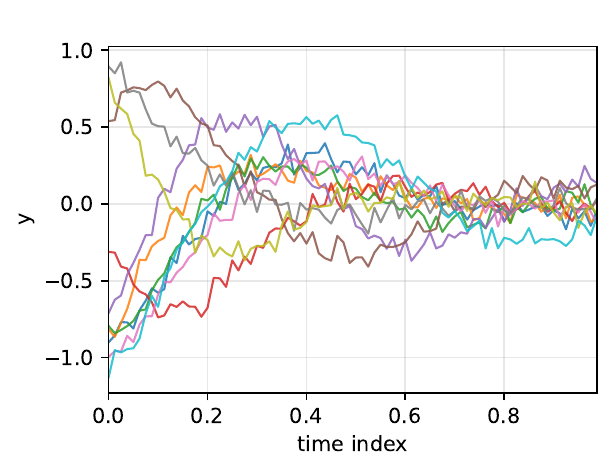}
        \caption*{\textbf{(c)} Damped Sinusoidal}
    \end{minipage}
    \caption{Illustrations of the synthetic datasets used for training. 
    Each dataset highlights a distinct modeling challenge: mode ambiguity, frequency drift, and amplitude drift.}
    \label{fig:datasets}
\end{figure}

\section{Model Details and Hyperparameters}\label{app:hyper}

For each dataset, we train a diffusion denoiser with output dimension matching the data. 
For the denoiser we use a ScoreNet-style architecture composed of FiLM–residual MLP blocks, 
trained with the $\varepsilon$-prediction objective used in diffusion models. 
The reported last-layer parameter counts correspond to the regression head 
$\texttt{Linear}(\text{hidden}\!\rightarrow\!\text{data\_dim})$, 
which maps the hidden representation to the data dimension and includes bias parameters.

On the bimodal sinusoidal dataset, the network has \textbf{8,330} total parameters; 
the LLLA head is \texttt{net.out} with weight $10{\times}32$ and bias 
(\textbf{330} last-layer parameters). 

On both the chirp and damped sine datasets, the model has \textbf{72,496} total parameters; 
the LLLA head is $80{\times}128$ plus bias (\textbf{10,320} last-layer parameters). 
For the full-parameter variant on the chirp dataset, we use the randomized version of the algorithm 
with a subsampled parameter set of \textbf{4,412} directions to control cost. 

On the 2D multimodal grid dataset, the network totals \textbf{7,810} parameters, 
with a \textbf{66}-parameter last layer. 
This setup lets us contrast (i) BayesDiff predictive variance, 
(ii) last-layer epistemic rollouts, and 
(iii) subsampled projected epistemic rollouts under matched training and sampling protocols.
\cref{tab:tuning} summarizes the key experimental settings for each dataset, including input dimensionality, model size, diffusion parameters, and optimization details used throughout our experiments. We report both last-layer Laplace (LLLA) and full-parameter Fisher/Laplace variants. 

\begin{table}[t]
\centering
\caption{Model architectures, dataset sizes, and parameter counts for all datasets.}
\resizebox{\textwidth}{!}{%
\begin{tabular}{lcccc}
\toprule
 & \textbf{2D Grid} & \textbf{Bimodal Sinusoid (L=10)} & \textbf{Chirp (L=80)} & \textbf{Damped Sine (L=40)} \\
\midrule
\textbf{Dataset size}  &6000/mode  &5000  &8000  &8000  \\
\textbf{Input dimension ($d$)}             & 2        & 10       & 80       & 40       \\
\textbf{Model width}                       & 32       & 32       & 128      & 128      \\
\textbf{Diffusion steps ($T$)}             & 800      & 600      & 600      & 600      \\
\textbf{Beta schedule}                     & cosine   & cosine   & cosine   & cosine   \\
\textbf{Learning rate}                     & $5\times10^{-6}$ & $5\times10^{-4}$ & $5\times10^{-4}$ & $5\times10^{-4}$ \\
\textbf{Optimizer} $(\beta_1,\beta_2)$     & (0.9, 0.999) & (0.9, 0.999) & (0.9, 0.999) & (0.9, 0.999) \\
\textbf{Batch size}                        & 256      & 512      & 256      & 256      \\
\textbf{Total parameters ($p$)}                  & 7{,}810  & 4{,}218  & 72{,}496 & 72{,}496 \\
\textbf{Last-layer parameters}             & 66       & 330      & 10{,}320 & 10{,}320 \\

\bottomrule
\end{tabular}
}
\vspace{3pt}
\label{tab:tuning}
\end{table}

\section{Implementation Details}

We reuse the same architectures and datasets across all runs. Training follows standard DDPM practice: \textbf{AdamW} optimizer with cosine learning–rate decay, gradient clipping, and an exponential moving average (EMA) of parameters. Unless stated otherwise, evaluation uses EMA weights; the EMA decay is tuned per dataset in the range $[0.999, 0.9999]$. We use diffusion loss with log–SNR weighting across timesteps (placing more mass around mid–SNR), and sample training timesteps from the same distribution used by the loss. Data preprocessing and augmentations (normalization, \emph{etc.}) match the baseline DDPM setup.

\textbf{Noise schedule and sampler}
We adopt a cosine (or linear, matching the baseline) $\beta$-schedule; the corresponding $(a_t, b_t, \tilde\sigma_t)$ are computed exactly from the schedule definitions.

\textbf{MAP estimate and posterior approximation.}
After training, we take the MAP point $\hat{\boldsymbol{\theta}}$ (the final EMA weights unless noted) and approximate the local parameter posterior with a Gaussian using the \texttt{laplace-torch} (Laplace) library. Curvature is computed as generalized Gauss–Newton / empirical Fisher on the training loss; damping (prior precision) is tuned on a small validation split. The posterior is computed once at $\hat{\boldsymbol{\theta}}$ and reused for all timesteps.

\begin{figure}[!t]
    \vspace{0.5cm}
    \centering
    \begin{overpic}[width=0.98\linewidth] 
    {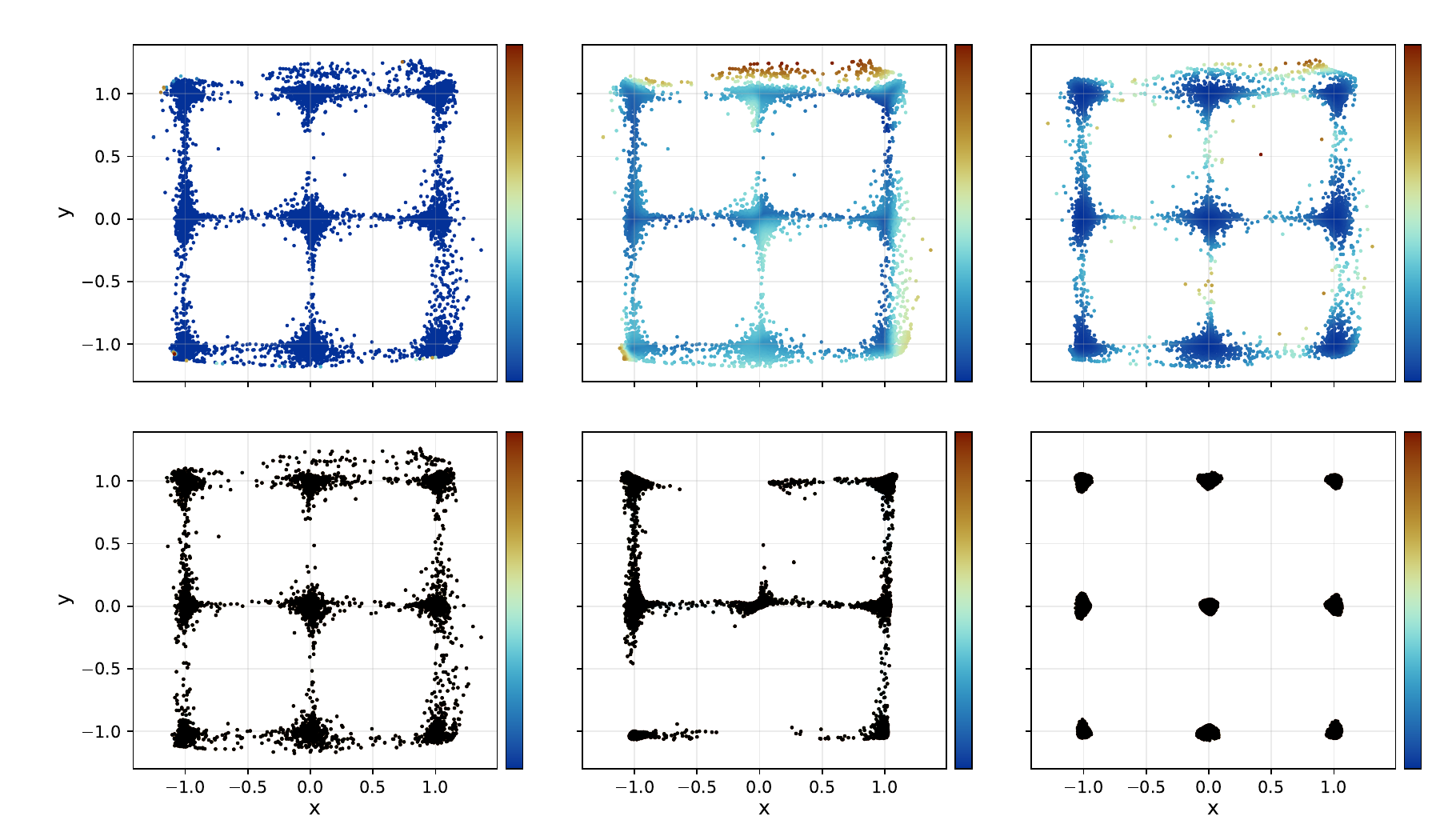}
            \put(35,0){\color{white}\rule{0.9em}{13em}} 
            \put(66,0){\color{white}\rule{0.9em}{13em}} 
            \put(97,0){\color{white}\rule{2em}{13em}} 
            \put(98,11){\rotatebox{90}{\scalebox{0.7}{filtered samples}}}
            \put(18,56){{\scalebox{0.85}{BayesDiff}}}
            \put(50,56){{\scalebox{0.85}{LLLA}}}
            \put(79,56){{\scalebox{0.85}{FLARE (ours)}}}
            \put(99,52){\rotatebox{90}{\scalebox{0.7}{certain}}}
            \put(99,31){\rotatebox{90}{\scalebox{0.7}{uncertain}}}
            
     \end{overpic}
    \caption{DDIM - Mode interpolation on a 2D Gaussian mixture adapted from~\citet{aithal2024understanding} and~\citet{jazbec2025generative}. The dataset consists of 9 Gaussian modes arranged on a square grid. The top row shows uncertainty scores from BayesDiff (left), LLLA (middle), and our method (right) for generated samples. The bottom row shows the same generated samples after filtering by a fixed uncertainty threshold. BayesDiff fails to assign high scores to points between modes, while LLLA performs even worse, leading to unreliable filtering. In contrast, our method produces faithful uncertainty estimates, enabling consistent removal of low-confidence samples.}
    \label{fig:grid_ddim}
\end{figure}

\begin{figure}[!t]
    \centering
    \includegraphics[width=0.98\linewidth]{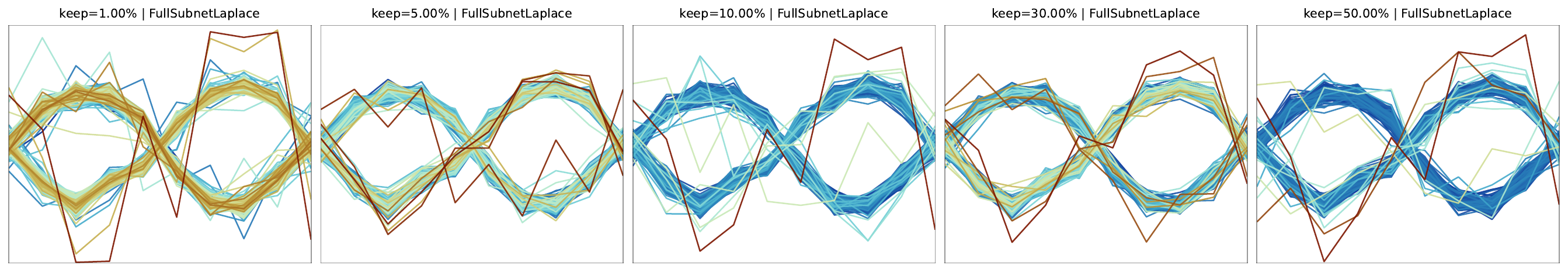}
    \includegraphics[width=0.98\linewidth]{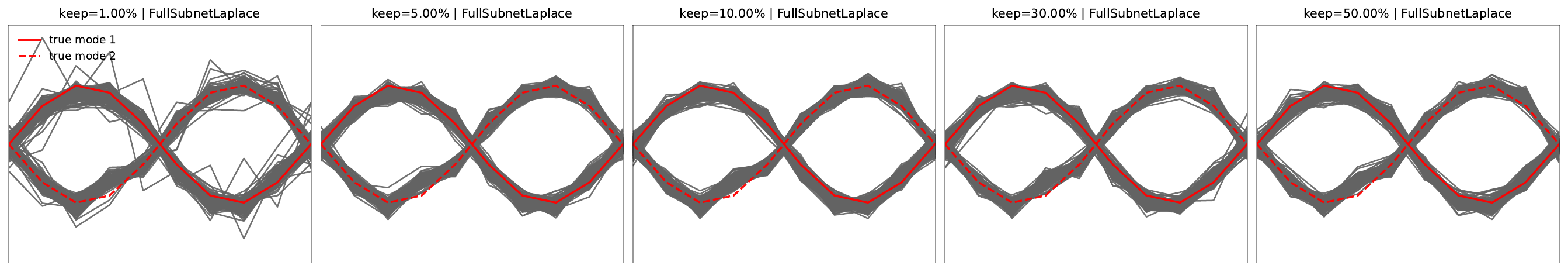}
    \caption{Parameter–budget ablation with \textsc{FullSubnetLaplace}. We retain
    $\{1\%, 5\%, 10\%, 30\%, 50\%\}$ of parameters. As the kept fraction increases,
    trajectories tighten around the data manifolds and the epistemic covariance
    $b_t^{2}\mathbf{J}_t\boldsymbol{\Sigma}_\theta\mathbf{J}_t^\top$ contracts smoothly;
    $10$–$30\%$ already yields stable mode coverage, with diminishing returns
    beyond $30$–$50\%$.}
    \label{fig:keep-ablation}
\end{figure}

\textbf{Jacobian handling.}
We never materialize $\mathbf{J}_t=\nabla_{\boldsymbol{\theta}}\varepsilon_{\boldsymbol{\theta}}(\mathbf{x}_t,t)$. All computations only require its action on vectors, obtained via autograd as vector--Jacobian and Jacobian--vector products (VJPs/JVPs), each costing a single forward$+$backward pass. We instantiate the posterior with the low-rank Laplace variant (via the \texttt{laplace} library), yielding $\boldsymbol{\Sigma}_\theta \approx \mathbf{U}\boldsymbol{\Lambda}\mathbf{U}^\top$ with rank $r\ll p$. The epistemic term is then formed without trace estimators:
\[
b_t^{2}\,\mathbf{J}_t \boldsymbol{\Sigma}_\theta \mathbf{J}_t^\top
~=~
b_t^{2}\sum_{i=1}^{r}\lambda_i\,(\mathbf{J}_t\mathbf{u}_i)(\mathbf{J}_t\mathbf{u}_i)^\top,
\]
computed by pushing the basis vectors $\{\mathbf{u}_i\}_{i=1}^r$ through $\mathbf{J}_t$ using $r$ JVPs per step $t$. We cache per-$t$ intermediates (e.g., $\varepsilon_{\hat{\boldsymbol{\theta}}}(\mathbf{x}_t,t)$ and $b_t$) to avoid redundant passes; the overall cost scales linearly with the chosen rank $r$ and is independent of the total parameter count $p$.


\section{Additional Experiments}

\paragraph{(a) DDIM results.}
We repeat our analysis with the deterministic DDIM sampler (setting the injected noise to zero, $\boldsymbol{\eta}_t \equiv \mathbf{0}$). In this case the aleatoric component vanishes and the one–step predictive covariance is dominated by the epistemic pushforward
\[
\mathrm{Cov}(\mathbf{x}_{t-1}\mid \mathbf{x}_t,t)
~\approx~
b_t^{2}\,\mathbf{J}_t\,\boldsymbol{\Sigma}_\theta\,\mathbf{J}_t^\top,
\]
composed with the DDIM step’s linear map when applicable. Qualitatively, the trajectories and uncertainty bands closely track the DDPM case: epistemic bands concentrate around the two ground–truth modes throughout the reverse sweep. \cref{fig:grid_ddim} shows results for 2D grid data. Again, it can be seen that FLARE faithfully filters the generated data.

\paragraph{(b) Parameter–budget ablation.}
We study the effect of restricting the parameter posterior to a fraction of weights using \textsc{FullSubnetLaplace}, keeping
$\{\!1\%, 5\%, 10\%, 30\%, 50\%\!\}$ of parameters.
As the kept fraction increases, samples tighten around the data manifolds and the epistemic covariance $b_t^{2}\mathbf{J}_t\boldsymbol{\Sigma}_\theta\mathbf{J}_t^\top$ contracts smoothly; already at $10$--$30\%$ we observe stable mode coverage, with diminishing returns beyond $30$--$50\%$. Representative trajectories across the grid are shown in \Cref{fig:keep-ablation}.

\begin{figure}[t]
\centering
\includegraphics[width=\linewidth]{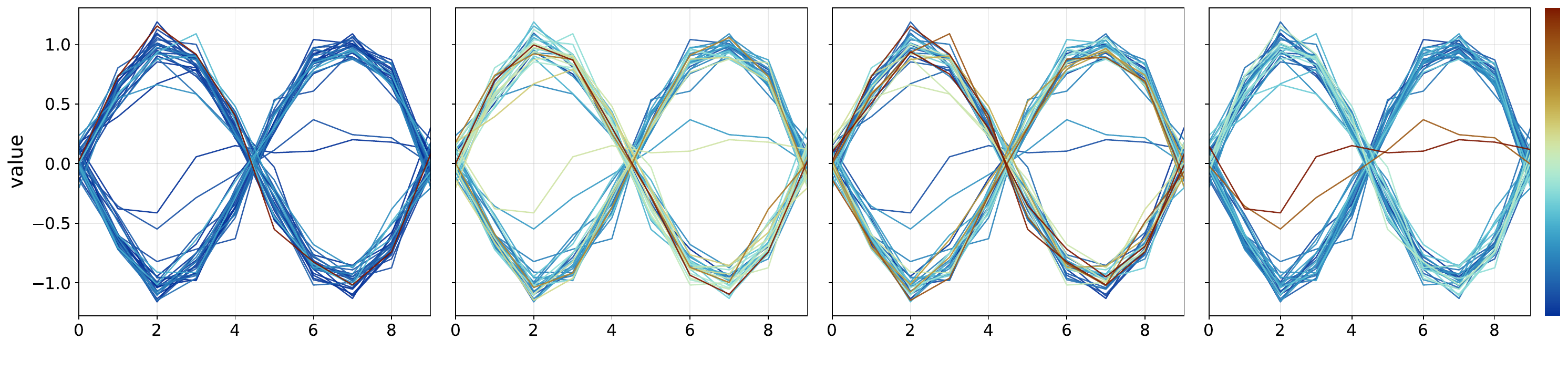}
\caption{\textbf{Four-way epistemic UQ comparison in the full-curvature regime.}
We compare uncertainty estimates from (left to right):
\emph{BayesDiff (last-layer Laplace, total recursion)},
\emph{Fuller-Hessian Laplace (Bayesdiff recursion)},
\emph{BayesDiff (epistemic-only recursion)},
and \emph{FLARE (Fisher--Laplace projection)}.
In this controlled setting, full-curvature computation is feasible and allows direct inspection
of epistemic structure.
The scalable approximation of~\cite{kou2023bayesdiff} (BayesDiff) already breaks down in this regime,
producing uncertainty maps that are not epistemically meaningful relative to the full-Hessian reference,
whereas FLARE closely tracks the curvature-based baseline while avoiding full-Hessian computation.}
\label{fig:fourway_uq}
\end{figure}

\section{Discussion around Parameter-Covariance Term}
\label{app:cross-term}

This appendix empirically evaluates the parameter cross-covariance term that arises in the full epistemic variance recursion and provides an independent Monte-Carlo (MC) validation of our analytic computation under a full-Hessian Laplace posterior.

\subsection{Theoretical Upper Bound on Cross-Covariance Term}
Let $\theta \sim \mathcal{N}(\hat\theta, \Sigma_\theta)$ denote the parameter posterior induced by a full-Hessian Laplace approximation. Linearizing the reverse diffusion step around $\hat\theta$ yields
\begin{equation}
x_{t-1} \;\approx\; a_t x_t - b_t \varepsilon_\theta(x_t,t),
\end{equation}
where $a_t,b_t$ are scalar schedule coefficients and $\varepsilon_\theta(\cdot)$ denotes the model’s noise predictor.

Conditioning on all non-parameter randomness (i.e.\ on a fixed reverse trajectory), the epistemic variance propagates as
\begin{equation}
\mathrm{Var}_\theta[x_{t-1}]
= a_t^2\,\mathrm{Var}_\theta[x_t]
+ b_t^2\,\mathrm{Var}_\theta[\varepsilon_t]
- 2 a_t b_t\, \mathrm{Cov}_\theta(x_t,\varepsilon_t).
\label{eq:epistemic-full}
\end{equation}
\paragraph{Bounding and absorbing the cross term.}
Condition on a fixed realization of aleatoric randomness $\eta$, so that all randomness is over
$\theta\sim\mathcal N(\hat\theta,\Sigma_\theta)$.
For any unit vector $u\in\mathbb S^{d-1}$, Cauchy--Schwarz yields
\begin{equation}
\Big|\mathrm{Cov}_\theta\!\big(u^\top x_t,\;u^\top \varepsilon_t\big)\Big|
\le
\sqrt{\mathrm{Var}_\theta(u^\top x_t)}\;\sqrt{\mathrm{Var}_\theta(u^\top \varepsilon_t)}.
\end{equation}
Applying Young's inequality $2\sqrt{ab}\le \delta a+\delta^{-1}b$ (for any $\delta>0$) gives
\begin{equation}
2\Big|\mathrm{Cov}_\theta\!\big(u^\top x_t,\;u^\top \varepsilon_t\big)\Big|
\le
\delta\,\mathrm{Var}_\theta(u^\top x_t)
+
\delta^{-1}\,\mathrm{Var}_\theta(u^\top \varepsilon_t).
\end{equation}
Taking the directional variance of \eqref{eq:epistemic-full} and bounding the cross term yields the sandwich bound
\begin{align}
\Big(a_t^2-|a_tb_t|\delta\Big)\mathrm{Var}_\theta(u^\top x_t)
+
\Big(b_t^2-|a_tb_t|\delta^{-1}\Big)\mathrm{Var}_\theta(u^\top \varepsilon_t)
\le
\mathrm{Var}_\theta(u^\top x_{t-1})
\\
\le
\Big(a_t^2+|a_tb_t|\delta\Big)\mathrm{Var}_\theta(u^\top x_t)
+
\Big(b_t^2+|a_tb_t|\delta^{-1}\Big)\mathrm{Var}_\theta(u^\top \varepsilon_t).
\nonumber
\end{align}
Thus the cross covariance term cannot introduce an uncontrolled contribution: it is always dominated by (and can be
absorbed into) the same two variance terms already present in the recursion.

\paragraph{When is the cross term small?}
Beyond being absorbable via Cauchy--Schwarz and Young's inequality, the cross term is
\emph{asymptotically negligible} under posterior concentration and local smoothness of the sampler.

Conditioned on a fixed realization of all aleatoric randomness, suppose the reverse trajectory
$x_t(\theta)$ admits a local linearization around the MAP,
\[
x_t(\theta)
=
x_t(\hat\theta)
+
G_t(\theta-\hat\theta)
+
r_t(\theta),
\]
where $G_t := \nabla_\theta x_t(\theta)\rvert_{\hat\theta}$ and the remainder satisfies
$\mathbb E_\theta \|r_t(\theta)\|^2 = o(\mathrm{tr}(\Sigma_\theta))$.
Then
\[
\mathrm{Cov}_\theta\!\bigl(x_t(\theta), J_t(\theta-\hat\theta)\bigr)
=
G_t \Sigma_\theta J_t^\top
+
o(\|\Sigma_\theta\|),
\]
and therefore
\[
\bigl\|\mathrm{Cov}_\theta(x_t(\theta), J_t(\theta-\hat\theta))\bigr\|
\;\le\;
\|G_t\|\,\|\Sigma_\theta\|\,\|J_t\|
+
o(\|\Sigma_\theta\|).
\]
Under posterior concentration ($\Sigma_\theta \to 0$), the cross term vanishes at the same order
as the leading epistemic variance term $J_t \Sigma_\theta J_t^\top$.

Importantly, $G_t$ is not an independent quantity: since the sampler depends on parameters only
through the denoiser, $G_t$ is itself a linear combination of past denoiser Jacobians
$\{J_s\}_{s>t}$ propagated through the linear reverse dynamics. Thus $\|G_t\|$ is controlled
by the same local smoothness and stability assumptions used to bound $\|J_t\|$.

A more direct sufficient condition is Lipschitz stability of the reverse trajectory in parameters:
assume that on the posterior mass,
\[
\|x_t(\theta)-x_t(\hat\theta)\| \le L_t \|\theta-\hat\theta\|.
\]
Then for any unit vector $u$,
\[
\mathrm{Var}_\theta(u^\top x_t)
\le
L_t^2\,\mathrm{tr}(\Sigma_\theta),
\qquad
\mathrm{Var}_\theta\!\bigl(u^\top J_t(\theta-\hat\theta)\bigr)
\le
\|J_t\|_2^2\,\lambda_{\max}(\Sigma_\theta),
\]
and by Cauchy--Schwarz,
\[
\bigl|\mathrm{Cov}_\theta(u^\top x_t, u^\top J_t(\theta-\hat\theta))\bigr|
\le
L_t \|J_t\|_2
\sqrt{\mathrm{tr}(\Sigma_\theta)\,\lambda_{\max}(\Sigma_\theta)}.
\]
Thus the cross term is small whenever the parameter posterior is concentrated and the reverse trajectory is stable to parameter perturbations.

\paragraph{Specialization to the linearized term.}
In our Monte-Carlo study we replace $\varepsilon_t$ by the linearized perturbation $J_t(\theta-\hat\theta)$.
In this case,
$\mathrm{Var}_\theta(u^\top J_t(\theta-\hat\theta)) = u^\top J_t\Sigma_\theta J_t^\top u$,
so the same absorption bound applies with
$\mathrm{Var}_\theta(u^\top \varepsilon_t)$ replaced by $u^\top J_t\Sigma_\theta J_t^\top u$.

Our main method (FLARE) drops the final cross term under a local decoupling assumption. Here we empirically verify that this term is numerically negligible.

\subsection{Empirical Monte Carlo Cross-Covariance Term Evaluation}

We empirically assess the ``missing'' cross-covariance term from the reviewer comment,
\begin{equation}
-2a_t b_t \, \mathrm{Cov}_{\theta}\!\bigl(x_t(\theta),\, J_t(\theta-\hat\theta)\bigr),
\label{eq:cross-term}
\end{equation}
where $\mathrm{Cov}_{\theta}(\cdot)$ denotes covariance over $\theta$ \emph{conditioned on a fixed realization of all aleatoric randomness} $\xi$
(initialization $x_T$ and diffusion noises $\{\eta_t\}$), as defined in our rebuttal.
We take $\hat\theta$ to be the MAP estimate and approximate the parameter posterior by a full-Hessian Laplace,
$\theta \sim \mathcal N(\hat\theta, \Sigma_\theta)$, with precision $P=\Sigma_\theta^{-1}$.

\paragraph{Shared stochastic reverse path.}
We fix a single realization $\xi$ and run stochastic DDPM sampling to obtain a trajectory $\{x_t(\hat\theta)\}_{t=0}^T$.
All Monte Carlo evaluations below reuse the same $\xi$, so that variability arises only from sampling $\theta$.

\paragraph{Local linearization term.}
At each step $t$, we compute the Jacobian at the MAP (evaluated along the shared path),
\begin{equation}
J_t \;=\; \nabla_\theta \varepsilon_\theta(x_t(\hat\theta), t)\big|_{\theta=\hat\theta},
\end{equation}
and use it to form the linearized perturbation $J_t(\theta-\hat\theta)$ for sampled $\theta$.

\paragraph{Monte Carlo estimator of the cross-covariance.}
We draw $S$ samples $\delta\theta^{(s)} \sim \mathcal N(0,\Sigma_\theta)$ and set $\theta^{(s)}=\hat\theta+\delta\theta^{(s)}$.
For each $s$, we recompute the DDPM reverse iterate $x_t(\theta^{(s)})$ using the \emph{same} aleatoric randomness $\xi$.
We then estimate the cross-covariance in~\eqref{eq:cross-term} dimensionwise:
\begin{equation}
\widehat{\mathrm{Cov}}_{\theta,j}(t)
=
\frac{1}{S-1}\sum_{s=1}^{S}
\Big(x_{t,j}(\theta^{(s)})-\bar x_{t,j}\Big)
\Big(\big(J_t\delta\theta^{(s)}\big)_j-\overline{(J_t\delta\theta)_j}\Big),
\label{eq:mc-cov}
\end{equation}
with $\bar x_{t,j}$ and $\overline{(J_t\delta\theta)_j}$ the sample means across $s$.

\paragraph{Epistemic recursion with and without the MC cross term.}
Using the same coefficients $(a_t,b_t)$ as in the main text, we compare:
\begin{align}
\mathrm{Var}^{\mathrm{noX}}_{t-1}
&= a_t^2\,\mathrm{Var}^{\mathrm{noX}}_{t}
  + b_t^2\,\mathrm{diag}\!\big(J_t \Sigma_\theta J_t^\top\big),
\\
\mathrm{Var}^{\mathrm{withX}}_{t-1}
&= a_t^2\,\mathrm{Var}^{\mathrm{withX}}_{t}
  + b_t^2\,\mathrm{diag}\!\big(J_t \Sigma_\theta J_t^\top\big)
  -2a_tb_t\,\widehat{\mathrm{Cov}}_{\theta}(t),
\end{align}
where $\widehat{\mathrm{Cov}}_{\theta}(t)\in\mathbb{R}^D$ stacks the dimensionwise estimates in~\eqref{eq:mc-cov}.
Per-sample epistemic scores are $u=\sum_{j=1}^{D}\mathrm{Var}_j$.

\paragraph{Result (baseline small model).}
On a baseline setting with $n=20$ shared-path samples and $S=128$ posterior draws per step,
the MC cross term changes the epistemic score by approximately $10^{-3}\%$--$10^{-2}\%$ on average.
A one-sided test against a practical threshold of $0.01\%$ rejects $H_0:\mathbb{E}[\Delta u/u]\ge 0.01\%$
with $p \approx 8.3\times 10^{-13}$, indicating the cross term is negligible at this scale.

\subsection{Quantitative Impact of the Cross Term}

Across $n=20$ samples on a shared stochastic DDPM trajectory, we observe that the relative contribution of the cross term is extremely small.

Specifically, the mean percent change
\[
\frac{u_{\mathrm{withX}} - u_{\mathrm{noX}}}{u_{\mathrm{noX}}}
\]
lies in the range
\[
10^{-3}\% \;\text{to}\; 5\times 10^{-3}\%,
\]
with the maximum observed per-sample contribution below $3\times 10^{-2}\%$.

We further test whether the mean contribution exceeds a practically meaningful tolerance. For a threshold of $\epsilon = 0.01\%$, we conduct a one-sided $t$-test:
\[
H_0:\; \mathbb{E}[\Delta u/u_{\mathrm{noX}}] \ge \epsilon,
\qquad
H_1:\; \mathbb{E}[\Delta u/u_{\mathrm{noX}}] < \epsilon,
\]
and reject $H_0$ with $p \approx 8\times 10^{-13}$. This confirms that the cross term is statistically and practically negligible.

\begin{table}[ht]
\centering
\caption{Effect size of the Monte Carlo cross-covariance term.
Statistics are computed over $N=80$ stochastic DDPM trajectories.
All values are reported in percent units.}
\label{tab:cross_effect}
\begin{tabular}{lccccc}
\toprule
Dataset & Mean (\%) & Std (\%) & Max (\%) & SE (\%) & df \\
\midrule
Sine  & 0.00338 & 0.00440 & 0.02220 & 0.00049 & 79 \\
Chirp & 0.00495 & 0.00610 & 0.02840 & 0.00068 & 79 \\
\bottomrule
\end{tabular}
\end{table}

\begin{table}[ht]
\centering
\caption{One-sided t-tests assessing whether the mean cross-covariance
contribution $\mu$ is below a user-chosen practical threshold $\tau$.
Negative t-statistics indicate evidence that $\mu < \tau$.}
\label{tab:cross_threshold}
\begin{tabular}{ccccccc}
\toprule
$\tau$ Threshold & Sine $t$ & Sine $p$ & Conclusion &
Chirp $t$ & Chirp $p$ & Conclusion \\
\midrule
0.005\% &
$-3.29$ & $0.9993$ & $\mu < 0.005\%$ &
$-0.10$ & $0.539$  & $\mu < 0.005\%$ \\

0.01\%  &
$-13.45$ & $1.0000$ & $\mu < 0.01\%$ &
$-7.43$  & $1.000$  & $\mu < 0.01\%$ \\
\bottomrule
\end{tabular}
\end{table}

\subsection{Implication for FLARE}

These experiments demonstrate that even under a full-Hessian Laplace posterior and explicit Jacobian evaluation, the parameter cross-covariance term contributes less than $10^{-2}\%$ to the epistemic uncertainty. Dropping this term therefore has negligible numerical impact while substantially simplifying the recursion and reducing computational cost, justifying the epistemic propagation used by FLARE.

%% file: references.bib
@article{meyer23,
  author       = {Raphael A. Meyer},
  title      = {Subspace Embedding via Leverage Score Sampling},
  year = 2023,
  journal = {RandNLA Proof Wiki},
  url          = {https://randnla.github.io/leverage-subspace-embedding/},
}

@article{Drineas17,
  title         = {Lectures on randomized numerical linear algebra},
  author={Drineas, Petros and Mahoney, Michael W},
  journal={American Mathematical Society},
  volume={3},
  number={4},
  year={2018},
  url={https://arxiv.org/pdf/1712.08880}
}

@article{murray2023randomized,
  title={Randomized numerical linear algebra: A perspective on the field with an eye to software},
  author={Murray, Riley and Demmel, James and Mahoney, Michael W and Erichson, N Benjamin and Melnichenko, Maksim and Malik, Osman Asif and Grigori, Laura and Luszczek, Piotr and Derezi{\'n}ski, Micha{\l} and Lopes, Miles E and others},
  journal={arXiv preprint arXiv:2302.11474},
  year={2023}
}

@article{erichson2019randomized,
  title={Randomized matrix decompositions using R},
  author={Erichson, N Benjamin and Voronin, Sergey and Brunton, Steven L and Kutz, J Nathan},
  journal={Journal of Statistical Software},
  volume={89},
  pages={1--48},
  year={2019}
}

@inproceedings{esser2024scaling,
  title={Scaling rectified flow transformers for high-resolution image synthesis},
  author={Esser, Patrick and Kulal, Sumith and Blattmann, Andreas and Entezari, Rahim and M{\"u}ller, Jonas and Saini, Harry and Levi, Yam and Lorenz, Dominik and Sauer, Axel and Boesel, Frederic and others},
  booktitle={Forty-first international conference on machine learning},
  year={2024}
}

@article{lu2022dpm,
  title={Dpm-solver: A fast ode solver for diffusion probabilistic model sampling in around 10 steps},
  author={Lu, Cheng and Zhou, Yuhao and Bao, Fan and Chen, Jianfei and Li, Chongxuan and Zhu, Jun},
  journal={Advances in neural information processing systems},
  volume={35},
  pages={5775--5787},
  year={2022}
}

@inproceedings{
liu2022flow,
title={Flow Straight and Fast: Learning to Generate and Transfer Data with Rectified Flow},
author={Xingchao Liu and Chengyue Gong and qiang liu},
booktitle={The Eleventh International Conference on Learning Representations },
year={2023},
url={https://openreview.net/forum?id=XVjTT1nw5z}
}

@inproceedings{
lipman2022flow,
title={Flow Matching for Generative Modeling},
author={Yaron Lipman and Ricky T. Q. Chen and Heli Ben-Hamu and Maximilian Nickel and Matthew Le},
booktitle={The Eleventh International Conference on Learning Representations },
year={2023},
url={https://openreview.net/forum?id=PqvMRDCJT9t}
}

@article{shu2024zero,
  title={Zero-shot uncertainty quantification using diffusion probabilistic models},
  author={Shu, Dule and Farimani, Amir Barati},
  journal={arXiv preprint arXiv:2408.04718},
  year={2024}
}

@article{daxberger2021laplace,
  title={Laplace redux-effortless bayesian deep learning},
  author={Daxberger, Erik and Kristiadi, Agustinus and Immer, Alexander and Eschenhagen, Runa and Bauer, Matthias and Hennig, Philipp},
  journal={Advances in neural information processing systems},
  volume={34},
  pages={20089--20103},
  year={2021}
}

@inproceedings{
kou2023bayesdiff,
title={BayesDiff: Estimating Pixel-wise Uncertainty in Diffusion via Bayesian Inference},
author={Siqi Kou and Lei Gan and Dequan Wang and Chongxuan Li and Zhijie Deng},
booktitle={The Twelfth International Conference on Learning Representations},
year={2024},
url={https://openreview.net/forum?id=YcM6ofShwY}
}

@inproceedings{de2025diffusion,
  title={Diffusion model guided sampling with pixel-wise aleatoric uncertainty estimation},
  author={De Vita, Michele and Belagiannis, Vasileios},
  booktitle={2025 IEEE/CVF Winter Conference on Applications of Computer Vision (WACV)},
  pages={3844--3854},
  year={2025},
  organization={IEEE}
}

@inproceedings{jazbec2025generative,
  title     = {Generative Uncertainty in Diffusion Models},
  author    = {Jazbec, Metod and Wong-Toi, Eliot and Xia, Guoxuan and Zhang, Dan and Nalisnick, Eric and Mandt, Stephan},
  booktitle = {Proceedings of the 41st Conference on Uncertainty in Artificial Intelligence (UAI)},
  year      = {2025},
  series    = {Proceedings of Machine Learning Research}
}

@inproceedings{
  berry2024casting,
  title={Shedding light on large generative networks: Estimating epistemic uncertainty in diffusion models},
  author={Berry, Lucas and Brando, Axel and Meger, David},
  booktitle={The 40th Conference on Uncertainty in Artificial Intelligence},
  year={2024}
}

@article{chan2024estimating,
  title={Estimating epistemic and aleatoric uncertainty with a single model},
  author={Chan, Matthew and Molina, Maria and Metzler, Chris},
  journal={Advances in Neural Information Processing Systems},
  volume={37},
  pages={109845--109870},
  year={2024}
}

@article{aithal2024understanding,
  title={Understanding hallucinations in diffusion models through mode interpolation},
  author={Aithal, Sumukh K and Maini, Pratyush and Lipton, Zachary and Kolter, J Zico},
  journal={Advances in Neural Information Processing Systems},
  volume={37},
  pages={134614--134644},
  year={2024}
}

@inproceedings{ho2020ddpm,
  title     = {Denoising Diffusion Probabilistic Models},
  author    = {Ho, Jonathan and Jain, Ajay and Abbeel, Pieter},
  booktitle = {Advances in Neural Information Processing Systems},
  volume    = {33},
  pages     = {6840--6851},
  year      = {2020}
}

@inproceedings{kendall2017uncertainties,
  title     = {What Uncertainties Do We Need in Bayesian Deep Learning for Computer Vision?},
  author    = {Kendall, Alex and Gal, Yarin},
  booktitle = {Advances in Neural Information Processing Systems},
  volume    = {30},
  year      = {2017}
}

@inproceedings{depeweg2018decomposition,
  title     = {Decomposition of Uncertainty in Bayesian Deep Learning for Efficient and Risk-Sensitive Learning},
  author    = {Depeweg, Stefan and Hern{\'a}ndez-Lobato, Jos{\'e} Miguel and Doshi-Velez, Finale and Udluft, Stephanie},
  booktitle = {International Conference on Machine Learning},
  pages     = {1184--1193},
  year      = {2018},
  organization = {PMLR}
}

@inproceedings{
song2020score,
title={Score-Based Generative Modeling through Stochastic Differential Equations},
author={Yang Song and Jascha Sohl-Dickstein and Diederik P Kingma and Abhishek Kumar and Stefano Ermon and Ben Poole},
booktitle={International Conference on Learning Representations},
year={2021},
url={https://openreview.net/forum?id=PxTIG12RRHS}
}

@article{hullermeier2021aleatoric,
  title   = {Aleatoric and Epistemic Uncertainty in Machine Learning: An Introduction to Concepts and Methods},
  author  = {H{\"u}llermeier, Eyke and Waegeman, Willem},
  journal = {Machine Learning},
  volume  = {110},
  pages   = {457--506},
  year    = {2021}
}

@inproceedings{ritter2018scalable,
  title     = {A Scalable Laplace Approximation for Neural Networks},
  author    = {Ritter, Hippolyt and Botev, Aleksandar and Barber, David},
  booktitle = {International Conference on Learning Representations},
  year      = {2018}
}

@inproceedings{lakshminarayanan2017deep,
  title     = {Simple and Scalable Predictive Uncertainty Estimation using Deep Ensembles},
  author    = {Lakshminarayanan, Balaji and Pritzel, Alexander and Blundell, Charles},
  booktitle = {Advances in Neural Information Processing Systems},
  volume    = {30},
  year      = {2017}
}

@inproceedings{gal2016dropout,
  title        = {Dropout as a Bayesian Approximation: Representing Model Uncertainty in Deep Learning},
  author       = {Gal, Yarin and Ghahramani, Zoubin},
  booktitle    = {International Conference on Machine Learning},
  pages        = {1050--1059},
  year         = {2016},
  organization = {PMLR}
}

@article{mackay1992practical,
  title   = {A Practical Bayesian Framework for Backpropagation Networks},
  author  = {MacKay, David J. C.},
  journal = {Neural Computation},
  volume  = {4},
  number  = {3},
  pages   = {448--472},
  year    = {1992}
}

@article{woodruff2014sketching,
  title     = {Sketching as a Tool for Numerical Linear Algebra},
  author    = {Woodruff, David P.},
  journal   = {Foundations and Trends in Theoretical Computer Science},
  volume    = {10},
  number    = {1--2},
  pages     = {1--157},
  year      = {2014},
  publisher = {Now Publishers Inc.}
}

@article{maddox2019simple,
  title={A simple baseline for bayesian uncertainty in deep learning},
  author={Maddox, Wesley J and Izmailov, Pavel and Garipov, Timur and Vetrov, Dmitry P and Wilson, Andrew Gordon},
  journal={Advances in neural information processing systems},
  volume={32},
  year={2019}
}

@article{krueger2017bayesian,
  title={Bayesian hypernetworks},
  author={Krueger, David and Huang, Chin-Wei and Islam, Riashat and Turner, Ryan and Lacoste, Alexandre and Courville, Aaron},
  journal={arXiv preprint arXiv:1710.04759},
  year={2017}
}

@article{hirschfeld2020uncertainty,
  title={Uncertainty quantification using neural networks for molecular property prediction},
  author={Hirschfeld, Lior and Swanson, Kyle and Yang, Kevin and Barzilay, Regina and Coley, Connor W},
  journal={Journal of Chemical Information and Modeling},
  volume={60},
  number={8},
  pages={3770--3780},
  year={2020},
  publisher={ACS Publications}
}

@article{oberdiek2022uqgan,
  title={Uqgan: A unified model for uncertainty quantification of deep classifiers trained via conditional gans},
  author={Oberdiek, Philipp and Fink, Gernot and Rottmann, Matthias},
  journal={Advances in Neural Information Processing Systems},
  volume={35},
  pages={21371--21385},
  year={2022}
}

@article{papamakarios2021normalizing,
  title={Normalizing flows for probabilistic modeling and inference},
  author={Papamakarios, George and Nalisnick, Eric and Rezende, Danilo Jimenez and Mohamed, Shakir and Lakshminarayanan, Balaji},
  journal={Journal of Machine Learning Research},
  volume={22},
  number={57},
  pages={1--64},
  year={2021}
}

@article{salimans2016improved,
  title={Improved techniques for training gans},
  author={Salimans, Tim and Goodfellow, Ian and Zaremba, Wojciech and Cheung, Vicki and Radford, Alec and Chen, Xi},
  journal={Advances in neural information processing systems},
  volume={29},
  year={2016}
}

@article{bohm2019uncertainty,
  title={Uncertainty quantification with generative models},
  author={B{\"o}hm, Vanessa and Lanusse, Fran{\c{c}}ois and Seljak, Uro{\v{s}}},
  journal={arXiv preprint arXiv:1910.10046},
  year={2019}
}

@inproceedings{nichol2021improved,
  title={Improved Denoising Diffusion Probabilistic Models},
  author={Nichol, Alex and Dhariwal, Prafulla},
  booktitle={International Conference on Machine Learning},
  year={2021}
}

@article{chewi2025ddpm,
  title={DDPM Score Matching and Distribution Learning},
  author={Chewi, Sinho and Kalavasis, Alkis and Mehrotra, Anay and Montasser, Omar},
  journal={arXiv preprint arXiv:2504.05161},
  year={2025}
}

@article{chen2022sampling,
  title={Sampling is as easy as learning the score: theory for diffusion models with minimal data assumptions},
  author={Chen, Sitan and Chewi, Sinho and Li, Jerry and Li, Yuanzhi and Salim, Adil and Zhang, Anru R},
  journal={arXiv preprint arXiv:2209.11215},
  year={2022}
}

@article{song2020denoising,
  title={Denoising diffusion implicit models},
  author={Song, Jiaming and Meng, Chenlin and Ermon, Stefano},
  journal={arXiv preprint arXiv:2010.02502},
  year={2020}
}

@article{lipman2024flow,
  title={Flow matching guide and code},
  author={Lipman, Yaron and Havasi, Marton and Holderrieth, Peter and Shaul, Neta and Le, Matt and Karrer, Brian and Chen, Ricky TQ and Lopez-Paz, David and Ben-Hamu, Heli and Gat, Itai},
  journal={arXiv preprint arXiv:2412.06264},
  year={2024}
}

@inproceedings{hang2023efficient,
  title={Efficient diffusion training via min-snr weighting strategy},
  author={Hang, Tiankai and Gu, Shuyang and Li, Chen and Bao, Jianmin and Chen, Dong and Hu, Han and Geng, Xin and Guo, Baining},
  booktitle={Proceedings of the IEEE/CVF international conference on computer vision},
  pages={7441--7451},
  year={2023}
}
